\documentclass[letterpaper]{article}
\usepackage[utf8]{inputenc}
\usepackage[T1]{fontenc}
\usepackage{lmodern}
\usepackage{arxiv}
\usepackage{microtype}
\usepackage[authoryear,round]{natbib}
\setcitestyle{citesep={;},aysep={,},yysep={;}}
\usepackage{amsmath,amssymb}
\usepackage{amsthm}
\theoremstyle{plain}
\newtheorem{theorem}{Theorem}[section]
\newtheorem{lemma}[theorem]{Lemma}
\theoremstyle{definition}
\newtheorem{assumption}[theorem]{Assumption}
\usepackage{graphicx}
\usepackage{tikz}
\usepackage{booktabs}
\usepackage{algorithm}
\usepackage{algpseudocode}
\usepackage{float}
\usepackage{placeins}
\usepackage{needspace}
\usepackage{hyperref}
\usepackage{url}
\DeclareMathOperator{\relu}{ReLU}
\usepackage{soul}

\title{BitSparse: Compact Activation Storage for \texorpdfstring{$\relu$}{ReLU} Feed-Forward Networks}

\author{%
    Max Zhu\textsuperscript{1}\thanks{Preprint. Corresponding author: Max Zhu, \href{mailto:mz406@cam.ac.uk}{\texttt{mz406@cam.ac.uk}}.\newline Code: \url{https://github.com/Maxzhu123/Bitmask_Sparsity}.}\quad
    Bohang Sun\textsuperscript{1}\quad
    Pietro Lio\textsuperscript{1}\quad
    Jialin Yu\textsuperscript{2}\\
    {\normalfont\textsuperscript{1}University of Cambridge\qquad
    \textsuperscript{2}University of Oxford}
}
\date{}
\renewcommand{\shorttitle}{BitSparse: Compact Activation Storage}
\renewcommand{\headeright}{}
\renewcommand{\undertitle}{}
\hypersetup{
    pdftitle={BitSparse: Compact Activation Storage for ReLU Feed-Forward Networks},
    pdfauthor={Max Zhu, Bohang Sun, Pietro Lio, Jialin Yu},
    hidelinks
}

\begin{document}
\maketitle

\begin{abstract}
Reducing memory use in large language model (LLM) training is an ongoing challenge. Activation checkpointing saves memory by recomputation, while lossy compression changes the stored activations. The reemergence of sparse $\relu$ and $\relu^2$ activations creates a third opportunity: losslessly compressing the activation cache itself. We present BitSparse, a GPU-resident lossless compression format for sparse activations. Compression and reconstruction run in custom Triton kernels fused into an autograd function, shrinking the activation cache while recovering it exactly for the backward pass. Bfloat16 (BF16) and FP8 activations are supported. Optional sign-bit packing further increases compression up to close to the theoretical limit. We analyze its compression ratio and compare it to the theoretical optimum. BitSparse lowers per-token VRAM by 31\% and 30\% Modded-NanoGPT and Nemotron-H respectively, with overheads of no more than 1\% at practical sequence lengths. Additionally, a new preallocated buffer system avoids host synchronization during training when variable size tensors are allocated. 
\end{abstract}

\section{Introduction}
\label{sec:introduction}
\label{sec:training-memory}

The memory required to train large language models limits the batch sizes and sequence lengths that fit on an accelerator. Alongside weights, gradients, and optimizer states, backpropagation requires activations retained from the forward pass. Parameter-related storage is largely independent of input size, whereas saved activations and some temporary buffers grow with the number of tokens processed simultaneously. Reducing this input-dependent memory cost is therefore important even when the model weights already fit on the device.

Feed-forward layers are a natural target. They expand each token's hidden representation, typically by a factor of $3$--$5$, and retain intermediate values for computing weight gradients. While the original Transformer \cite{attention_all_need} used quadratic memory, more recent memory-efficient attention \citep{attention_memory, flash_attention} and state-space architectures such as Mamba \citep{mamba} are both linear in the number of tokens, so the feedforward layers make up a significant portion of the memory cost. 

\suppressfloats[t]
\begin{figure}[t]
    \centering
%
%
\begingroup
\definecolor{figInk}{HTML}{23364D}
\definecolor{figBlue}{HTML}{346A9A}
\definecolor{figTeal}{HTML}{287E80}
\definecolor{figValue}{HTML}{7853A0}
\definecolor{figSign}{HTML}{BD493D}
\definecolor{panelTrain}{HTML}{F5F0E7}
\definecolor{panelCache}{HTML}{EDF3FA}
\definecolor{panelGPU}{HTML}{ECF4EF}
\setlength{\fboxsep}{5pt}
%
\colorbox{panelTrain}{%
\begin{minipage}[t][5.2cm][t]{\dimexpr0.22\linewidth-2\fboxsep\relax}
    \vspace{0pt}\raggedright\footnotesize
    \parbox[t][0.45cm]{\linewidth}{\raggedright\textbf{Training}}\par
    \medskip
    \centering
    \begin{tikzpicture}[
        x=1cm, y=1cm,
        card/.style={draw=figBlue!45,fill=white,rounded corners=2pt,
                     line width=0.5pt},
        title/.style={font=\scriptsize\bfseries,text=figInk,
                      text width=2.1cm,align=center,inner sep=0pt},
        body/.style={font=\scriptsize,text=figInk,
                     text width=2.1cm,align=center,inner sep=0pt},
        micro/.style={font=\tiny,text=figInk!85,
                      text width=2.1cm,align=center,inner sep=0pt},
        bitOn/.style={draw=figBlue!65,fill=figBlue!7,minimum height=0.30cm,
                      inner sep=1pt,font=\ttfamily\tiny,line width=0.4pt},
        slotOn/.style={draw=figBlue!70,fill=figBlue!45,line width=0.3pt},
        slotOff/.style={draw=figInk!30,fill=figInk!10,line width=0.3pt}
    ]
        \draw[card] (-1.15,0) rectangle (1.15,-1.06);
        \node[micro] at (0,-0.28)
            {$\relu$ activations are cached};
        \node[body] at (0,-0.74) {$H=\relu(Z)$};
        \node[micro,text=figInk,text width=2.7cm] at (0,-1.30)
            {GPU memory is wasted:};
        \draw[card] (-1.15,-1.56) rectangle (1.15,-2.96);
        \node[title] at (0,-1.94)
            {(a) Activations:\\$\approx$93\% sparsity};
        \foreach \i in {1,3,4}{%
            \filldraw[slotOff] ({-1.08+\i*0.44},-2.28)
                rectangle ++(0.40,-0.28);
        }
        \foreach \i in {0,2}{%
            \filldraw[slotOn] ({-1.08+\i*0.44},-2.28)
                rectangle ++(0.40,-0.28);
        }
        \node[micro] at (0,-2.76) {storage is wasted};
        \draw[card] (-1.15,-3.06) rectangle (1.15,-4.45);
        \node[title] at (0,-3.44)
            {(b) Positive values\\have sign bit 0};
        \node[bitOn,minimum width=0.30cm] (sgn) at (-0.46,-3.92)
            {\textcolor{figSign}{0}};
        \node[bitOn,minimum width=0.86cm] (pay) at (0.18,-3.92) {0110110};
        \node[micro] at (0,-4.24) {the sign bit is redundant};
    \end{tikzpicture}
\end{minipage}%
}\hfill
%
\colorbox{panelCache}{%
\begin{minipage}[t][5.2cm][t]{\dimexpr0.54\linewidth-2\fboxsep\relax}
    \vspace{0pt}\raggedright\footnotesize
    \parbox[t][0.45cm]{\linewidth}{\raggedright\textbf{BitSparse Activation Compression}}\par
    \medskip
    \centering
    \resizebox{\linewidth}{!}{\input{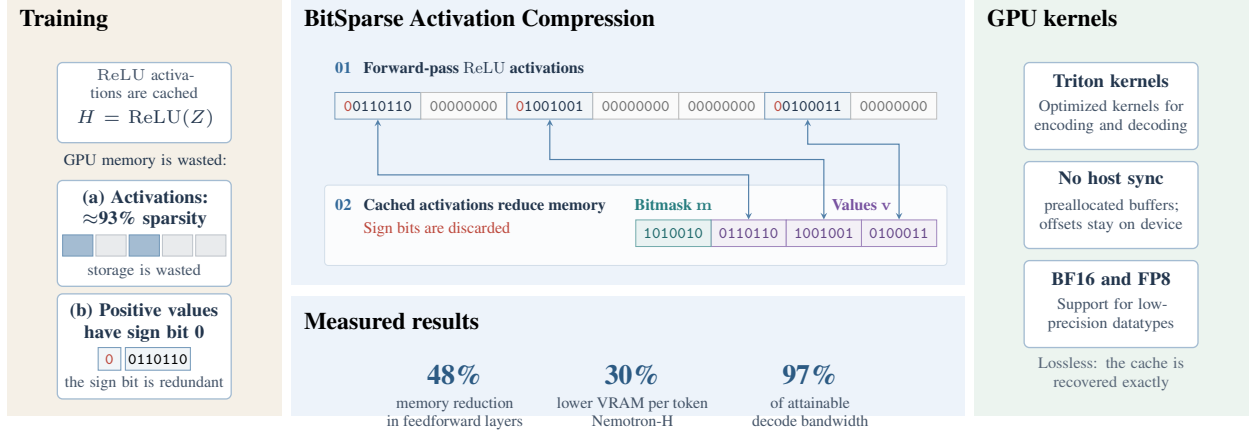}}
    \par\vspace{0.05cm}%
    \noindent\makebox[\linewidth][c]{%
        \textcolor{white}{\rule{\dimexpr\linewidth+2\fboxsep\relax}{0.14cm}}}%
    \par\vspace{0.08cm}\par
    \raggedright
    \parbox[t][0.45cm]{\linewidth}{\raggedright\textbf{Measured results}}\par
    \medskip
    \centering
    \begin{tikzpicture}[
        x=1cm, y=1cm,
        stat/.style={font=\fontsize{10.5}{11.5}\selectfont\bfseries,
                     text=figBlue!85!black,inner sep=0pt,anchor=north,
                     align=center},
        key/.style={font=\tiny,text=figInk!85,inner sep=0pt,
                    anchor=north,align=center}
    ]
        \foreach \x/\val/\first/\second in {%
            -2.35/48\%/{memory reduction}/{in feedforward layers},
            0/30\%/{lower VRAM per token}/{Nemotron-H},
            2.35/97\%/{of attainable}/{decode bandwidth}} {
            \node[stat] at (\x,0) {\val};
            \node[key] at (\x,-0.42) {\first\\\second};
        }
    \end{tikzpicture}
\end{minipage}%
}\hfill
\colorbox{panelGPU}{%
\begin{minipage}[t][5.2cm][t]{\dimexpr0.22\linewidth-2\fboxsep\relax}
    \vspace{0pt}\raggedright\footnotesize
    \parbox[t][0.45cm]{\linewidth}{\raggedright\textbf{GPU kernels}}\par
    \medskip
    \centering
    \begin{tikzpicture}[
        x=1cm, y=1cm,
        card/.style={draw=figBlue!45,fill=white,rounded corners=2pt,
                     line width=0.5pt},
        title/.style={font=\scriptsize\bfseries,text=figInk,
                      text width=2.1cm,align=center,inner sep=0pt},
        body/.style={font=\tiny,text=figInk!85,
                     text width=2.1cm,align=center,inner sep=0pt}
    ]
        \draw[card] (-1.15,0) rectangle (1.15,-1.15);
        \node[title] at (0,-0.24) {Triton kernels};
        \node[body] at (0,-0.72) {Optimized kernels for \\ encoding and decoding};
        \draw[card] (-1.15,-1.31) rectangle (1.15,-2.46);
        \node[title] at (0,-1.55) {No host sync};
        \node[body] at (0,-2.03) {preallocated buffers;\\offsets stay on device};
        \draw[card] (-1.15,-2.62) rectangle (1.15,-3.77);
        \node[title] at (0,-2.86) {BF16 and FP8};
        \node[body] at (0,-3.34) {Support for low-precision datatypes};
        \node[body,text=figInk!75,text width=2.5cm] at (0,-4.12)
            {Lossless: the cache is\\recovered exactly};
    \end{tikzpicture}
\end{minipage}%
}
\endgroup
    \caption{Overview of BitSparse. Caching $\relu$ activations for backpropagation is expensive: most entries are exact zeros, and every stored value carries a redundant sign bit. BitSparse caches only the nonzero values and reconstructs them with optimized Triton kernels.}
    \label{fig:bitsparse-codec}
\end{figure}

Certain activations produce sparse outputs that can be exploited to reduce memory usage. The ReLU activation ($\relu$) maps nonpositive inputs to exact zeros \citep{relu}. Large language models (LLMs) have used alternative activations such as GELU \cite{gelu} and gated activations such as SwiGLU. However, recently squared ReLU ($\relu^2$) activations have demonstrated competitive efficiency in LLMs \citep{relu2} and are used in the Nemotron family and Modded-NanoGPT \citep{nemotron_h, nemotron_nano, nemotron_4, modded_nanogpt}. Sparsity naturally emerges in training without any sparsity-inducing objective \citep{lazy_neuron, relu_strikes_back}. In our measurements, the pretrained Nemotron-H model has approximately $92\%$ feed-forward activation sparsity.

Exploiting this redundancy requires an inexpensive enough representation to use during backpropagation with minimal latency. Alternatives such as activation checkpointing reduce storage by recomputing forward intermediates \citep{act_checkpointing}, while lossy activation compression trades accuracy for memory \citep{actnn, CompAct}. This raises a simple question: \textbf{can the exact zeros already produced by ReLU-family LLMs be converted into activation-memory savings without recomputation or approximation}?

We present \emph{BitSparse}, a GPU-native format for lossless compression of sparse activations using a bitmask and compact nonzero values. Custom Triton kernels compress activations during the forward pass and reconstruct them as needed during backward. A shared value buffer avoids host synchronization. BitSparse supports the $\relu$ and $\relu^2$ activations with optimized kernels for minimum overhead. Figure~\ref{fig:bitsparse-codec} illustrates the approach. Our contributions are:
\begin{itemize}
    \item A GPU activation-storage scheme for ReLU and squared ReLU, with tile-parallel reconstruction, integration into feed-forward backpropagation and a preallocated buffer to avoid host synchronization.
    \item An information-theoretic analysis of sparse activation storage, together with compression ratios and break-even sparsity thresholds for the implemented representation.
    \item Evaluations spanning activation sparsity, codec throughput, MLP comparisons, and language-model forward/backward workloads. BitSparse reduces peak allocated memory by up to $48\%$ in an MLP, and reduces the memory per token by approximately $31\%$ for Modded-NanoGPT and $30\%$ for Nemotron-H, with minimal runtime overhead.
\end{itemize}

\section{Background}
\label{sec:background}

\subsection{\texorpdfstring{$\relu$}{ReLU} activations and their gradients}
\label{sec:relu_networks}
We review backpropagation through a $\relu^2$ activation followed by a linear projection; the $\relu$ derivation is given in Appendix~\ref{app:relu-gradients}. The input is $Z \in \mathbb{R}^{M \times N}$ and the weight is $W \in \mathbb{R}^{D \times N}$, where $M$ is the number of activation rows (for example, tokens flattened across a batch), and $N$ and $D$ are the input and output feature dimensions, respectively. Let $\relu(z)=\max(0,z)$ act elementwise, $G=\partial\mathcal{L}/\partial Y \in \mathbb{R}^{M \times D}$ be the upstream gradient, and $\odot$ denote elementwise multiplication. For the forward pass
\begin{equation}
    H=\relu(Z),
    \qquad
    \tilde{H}=\relu^2(Z)=H^2,
    \qquad
    Y=\tilde{H}W^\top,
\end{equation}
the gradients are
\begin{equation} \label{eq:relu2_backward}
    \frac{\partial\mathcal{L}}{\partial W} = G^\top \tilde{H},
    \qquad
    \frac{\partial\mathcal{L}}{\partial \tilde{H}} = GW,
    \qquad
    \frac{\partial\mathcal{L}}{\partial Z} = 2\frac{\partial\mathcal{L}}{\partial \tilde{H}} \odot\mathbb{I}[Z>0] \odot Z,
\end{equation}
The activation-dependent information required for these gradients is contained in the positive values of $Z$ and their positions. The same principle applies to differentiable pointwise functions of $\relu(Z)$, such as $\relu^3(Z)$ or $\sin(\relu(Z))$, although these are less popular.

\subsection{Entropy of sparse \texorpdfstring{$\relu$}{ReLU} activations}
\label{sec:background-entropy}
We wish to losslessly compress the activations $H$. Let $L=MN$ be the number of logical entries, excluding padding, and let $s=\|H\|_0$ denote the number of nonzero values. We denote the zero fraction, or \emph{activation sparsity}, by $\rho = 1-s/L$. 
Let $d$ denote the storage width of the numerical format in bits, with $d=16$ for BF16 and $d=8$ for FP8. Each positive value has a redundant sign bit, leaving $d-1$ non-sign bits.

For independent symbols, the average lossless code length per symbol can approach the Shannon entropy for long sequences. The following idealized model is used to obtain a reference entropy.

\begin{assumption}[Abstract sparse payload model] \label{ass:activation_distribution}
    The entries of $H$ are independent and identically distributed abstract symbols. Each entry is the distinguished zero symbol with probability $\rho$, where $0<\rho<1$, and otherwise carries a uniformly distributed $(d-1)$-bit payload. The $2^{d-1}$ payload symbols are distinct from the zero symbol. 
\end{assumption}

All logarithms are base two and $h_2$ denotes the binary entropy function; Appendix~\ref{app:proofs} gives its definition together with the proof of the following result.
\begin{theorem}[Entropy of sparse activations]
\label{thm:sparse-entropy}
Under Assumption~\ref{ass:activation_distribution}, the entropy and optimum asymptotic compression ratio relative to dense $d$-bit storage are
\begin{equation}
    \mathcal{H}(H)=L\bigl[h_2(\rho)+(1-\rho)(d-1)\bigr],
    \qquad R(\rho)=\frac{dL}{\mathcal{H}(H)}.
    \label{eq:ideal-compression}
\end{equation}
\end{theorem}

The entropy per entry consists of the support entropy $h_2(\rho)$ and the expected nonzero payload cost $(1-\rho)(d-1)$. 
Lemma~\ref{lem:entropy-upper-bound} shows that this is an upper bound for actual nonnegative floating-point activations: restricting payloads to valid positive values, using nonuniform distributions, or introducing correlations across positions can only lower the entropy.

\section{Methodology}
\label{sec:methodology}

\subsection{Overview}
We propose \emph{BitSparse}, an activation compression method for $\relu$-like activations (e.g. $\relu^2$) that adopts the notation and storage conventions of Section~\ref{sec:background-entropy}. BitSparse stores only the nonzero values, together with small index buffers for reconstruction, so the activation cache shrinks while the feed-forward projection is still computed densely. Activations are compressed before caching and reconstructed during the backward pass, preserving values in the selected storage format. 

\subsection{Tile-wise sparse activation format}
\label{sec:method-tile-format}
BitSparse is a tile-based sparse representation for moderately sparse tensors. It partitions $H$ into a regular grid of $B \times B$ tiles, with $B=64$ by default, so the grid has
\begin{equation}
    G_M = \left\lceil \frac{M}{B} \right\rceil,
    \qquad
    G_N = \left\lceil \frac{N}{B} \right\rceil
    \label{eq:tile-grid}
\end{equation}
rows and columns. Tiles are ordered row-major, with tile $(u,v)$ identified by $t=uG_N+v$, giving $T=G_MG_N$ tiles. Boundary tiles are logically padded with zeros so every tile has $B^2=4096$ entries. The representation of $H$ consists of three arrays: a bitmask $\mathbf{m}$ with one bit per activation position (including padding), a compact value array $\mathbf{v}$ holding only positive values, and a 32-bit prefix array $\mathbf{p}$.
For a tile $t$, the bitmask marks the nonzero entries and the value array collects the corresponding positives; the full arrays concatenate the per-tile arrays in tile order, following row-major order within each tile,
\begin{align}
    \mathbf{m}_t &= \mathbb{I}[H_t>0] \in \{0, 1\}^{B \times B}, \\
    \mathbf{v}_t &= \operatorname{concat}\{h \in H_t : h > 0\} \in \mathbb{F}_d^{s_t}, \\
    \mathbf{m} &= \operatorname{concat}\{\mathbf{m}_t: t \in[0,\dots, T-1]\} \in \{0, 1\}^{TB^2}, \\
    \mathbf{v} &= \operatorname{concat}\{\mathbf{v}_t: t \in[0,\dots, T-1]\} \in \mathbb{F}_d^{s}.  \label{eq:value}
\end{align}
Here $\mathbb{F}_d$ denotes the storage format of Section~\ref{sec:background-entropy}, $s_t$ is the number of nonzero values in tile $t$, and $\sum_{t=0}^{T-1}s_t=s$. Together, $\mathbf{m}$ and $\mathbf{v}$ determine $H$ exactly. In memory the mask is packed into unsigned bytes, one bit per padded position, which avoids the eight bits per element of a PyTorch Boolean tensor.
To locate each tile's values directly, a prefix array holds their starting logical offsets,
\begin{equation}
    p_t=\sum_{q=0}^{t-1}s_q \quad (0\leq t<T),
    \qquad
    \mathbf{p} \in \mathbb{I}_{32}^T,
    \label{eq:prefix}
\end{equation}
where $\mathbb{I}_{32}$ is the set of 32-bit unsigned integers and the empty sum gives $p_0=0$. Consequently, each $64\times64$ tile requires 512 bytes of mask metadata and 32 bits of prefix metadata; the full sparse object additionally records the logical shape, tile geometry and stored dtype on the CPU.
\subsection{Sign-bit packing}
\label{sec:method-packing}
Because the output of $\relu$ activations is always nonnegative, our method can optionally remove the redundant sign bit. BF16 can be reduced to 15 bits and FP8 values to 7 bits. Zero entries are reconstructed as positive zero. We refer to BitSparse with sign-bit packing enabled as \emph{BitSparse-sb}.

GPU memory is byte-aligned, so a custom Triton kernel is written for packing individual bits. Let $u_i$ be the unsigned bit pattern of the $i$th compact value $v_i$ with its sign bit removed, and let $w_j$ be the $j$th output byte. The corresponding input index is $i=\lfloor 8j/(d-1)\rfloor$, with bit offset $k=8j\bmod(d-1)$ within $u_i$. The output byte is
\begin{equation} \label{eq:sign_bit_pack}
    w_j = \left[(u_i \gg k) \middle| \left(u_{i+1} \ll (d - k - 1)\right)\right] \;\&\; (11111111)_2.
\end{equation}
Here $i$ indexes compact values, $j$ indexes packed bytes, and $k$ is a bit offset; out-of-range input words are treated as zero. The sign-bit packing reads two words and writes one byte at a time without atomics for high speed.
Decoding works in an analogous way, also without atomics.


\subsection{GPU compression and reconstruction}
\label{sec:method-kernels}

Compression is performed entirely on the GPU in three stages. First, a Triton tile kernel loads one tile, evaluates $H_t>0$, constructs $\mathbf{m}$ and the count $s_t$. Second, a scan over the per-tile counts constructs the starting offsets $\mathbf{p}$ in Equation~\eqref{eq:prefix}. Third, a second Triton kernel expands each tile mask and scatters positive values into $\mathbf{v}$ according to Equation~\eqref{eq:value}. Optionally, a final Triton kernel strips the sign bits and packs values into bytes according to Equation~\eqref{eq:sign_bit_pack}. 

Reconstruction is performed by a single Triton kernel. Each tile identifies the location of its values using $p_t$ and loads the relevant values along with the bitmask. Next, a prefix sum is used to compute the location of each value in the final array. Finally, a masked gather is used to produce the decompressed tensor $H$, with sign-bit unpacking if necessary. Optionally, the values may be squared during decoding for the $\relu^2$ activation.

Two further kernels fuse the backward-pass operations. The $\relu$ backward pass requires only the support mask, so a single kernel multiplies the incoming gradient by that mask in place, without reconstructing the activation. The $\relu^2$ pass performs the same operation with an additional dependence on $Z$, again in a single fused kernel. Furthermore, our blockwise structure allows for computing the backward matmul (Equations~\ref{eq:relu-backward},~\ref{eq:relu2_backward}) in blocks to avoid materializing the full activation to reduce peak memory further.

\subsection{Shared value buffers and implementation details} \label{sec:method-buffer}
The size of the value buffer $\mathbf{v}$ depends on the data, so allocating it on the host requires reading the device-resident count $s$, which interrupts asynchronous kernel submission. To avoid this, BitSparse optionally appends the values of different activations to a \texttt{TensorBuffer}, a large preallocated block with device memory offsets. Compressing a tensor appends its values and saves the starting offset on the device; decoding reads from the buffer using that offset, implemented in the kernels as \texttt{val\_idx = tl.load(prefix\_ptr + tile\_id) + tl.load(vals\_offset\_ptr)}. This allows BitSparse to operate completely asynchronously. The buffer is reset at the end of every train step, and must be sized to not overflow. We recommend setting it as large as possible to use the unallocated memory. 

The buffer gives a latency improvement for small sizes which are CPU bound. Appendix~\ref{app:buffer-ablation} shows a $5.2\%$ runtime reduction at $32$ tokens, diminishing as batch size increases and becoming negligible beyond $128$ tokens. Furthermore, this allows BitSparse to work with \texttt{torch.compile}. 

\subsection{Sparse activation checkpointing for feed-forward layers}
\label{sec:method-ffn}
Section~\ref{sec:relu_networks} shows that the positive preactivation values and their positions contain the activation-dependent information required for backpropagation. We therefore cache a compressed representation of $\relu(Z)$ for both activation functions. For $\relu^2$, the unsquared activation is stored to avoid a more unstable square root and keep numerical consistency with the forward pass. 

Our method is implemented in PyTorch as a custom $\texttt{torch.autograd.Function}$ class that combines the $\relu$ or $\relu^2$ activation along with a linear matrix multiplication; Algorithm~\ref{alg:relu-linear} gives its forward and backward passes. Sign-bit packing and \texttt{TensorBuffer} are optional, with support for Bfloat16 and FP8 (both e5m2 and e4m3) with scaling \cite{fp8_scaling}. 

The overhead of compressing and decompressing is fairly minimal for typical LLM training. Besides using highly optimized Triton kernels, the overhead is asymptotically reduced as the LLM's hidden dimension increases. To see this, observe that the FLOP complexity for the linear layer $HW^\top$ is $\mathcal{O}(MND)$. Since the size of $H$ is $MN$, the ratio overhead from compressing $H$ to the matrix multiplication is $\mathcal{O}(1/D)$. 

\section{Theoretical compression analysis}
\label{sec:compression-analysis}

\begin{theorem}[BitSparse compression ratios]
\label{thm:bitsparse-compression}
For $B=64$ and negligible boundary padding and per-tensor overhead, the compression ratios without and with sign-bit packing are, respectively,
\begin{equation}
    R_{\mathrm{BitSparse}}\simeq\frac{d}{d(1-\rho)+1+1/128},
    \qquad
    R_{\mathrm{sb}}\simeq
    \frac{d}{(d-1)(1-\rho)+1+1/128}.
    \label{eq:bitsparse-asymptotic}
\end{equation}
\end{theorem}

To see this result, observe the storage format has three components. Each entry requires on average one mask bit, $32/B^2=1/128$ prefix bits and $(1-\rho)$ nonzero values at $d$ bits unpacked or $d-1$ bits packed. Dividing the dense cost $d$ by this gives the ratios above. A full derivation is given in Appendix~\ref{app:proofs}. This is exact when there no padding and alignment mismatches. 

Table~\ref{tab:theoretical-compression} compares the ratios with and without sign-bit packing against the ideal rates of Equation~\eqref{eq:ideal-compression}.
Compared with the ideal rate in Theorem~\ref{thm:sparse-entropy}, the BitSparse-sb has an excess cost per entry of
\begin{equation}
    \frac{S_{\mathrm{BS}}}{L}
    -\bigl[h_2(\rho)+(1-\rho)(d-1)\bigr]
    \simeq 1-h_2(\rho)+\frac{1}{128}.
    \label{eq:compression-gap}
\end{equation}
At $\rho=1/2$ the excess is exactly the prefix cost, $1/128$. Storage is reduced whenever sparsity exceeds $\rho>(1+1/128)/d$ without packing, or $\rho>1/[128(d-1)]$ with packing. These thresholds are small, roughly $6\%$ and $13\%$ for unpacked BF16 and FP8 and below $0.2\%$ once packed, so any realistic activation sparsity reduces storage. The benefit of sign-bit packing is most notable when there are many nonzero values, and almost entirely compensates for the extra bitmask storage.  

\begin{table}[t]
    \centering
    \caption{Theoretical compression ratios of BitSparse, BitSparse-sb, and the theoretical optimum under Assumption~\ref{ass:activation_distribution}. Boundary padding overheads are neglected.}
    \label{tab:theoretical-compression}
    \small
    \begin{tabular}{rcccccc}
        \toprule
        & \multicolumn{3}{c}{BF16} & \multicolumn{3}{c}{FP8} \\
        Sparsity $\rho$ & BitSparse & BitSparse-sb & Ideal & BitSparse & BitSparse-sb & Ideal \\
        \midrule
        $50\%$ & $1.776\times$ & $1.881\times$ & $1.882\times$ & $1.598\times$ & $1.775\times$ & $1.778\times$ \\
        $75\%$ & $3.195\times$ & $3.363\times$ & $3.508\times$ & $2.660\times$ & $2.901\times$ & $3.123\times$ \\
        $90\%$ & $6.135\times$ & $6.380\times$ & $8.126\times$ & $4.425\times$ & $4.684\times$ & $6.843\times$ \\
        \bottomrule
    \end{tabular}
\end{table}

\section{Related work}
\label{sec:related-work}

\textbf{Lossy activation compression and low-precision training.}
Multiple methods have investigated lossy compression for activations. ActNN \citep{actnn} uses randomized mixed-precision quantization and analyzes the resulting gradient variance. GACT \citep{gact} extends activation compression to generic network architectures, dynamically choosing tensor compression ratios. CompAct \citep{CompAct} uses random projections to store low-rank activations. 

Low precision datatypes such as FP8 halves the storage per value compared to BF16 and offers speed-ups with hardware support \citep{fp8_scaling}. COAT \citep{coat} targets both activation and optimizer-state memory through FP8 quantization. Hardware-accelerated 2:4 sparsity for $\relu^2$ activations has been investigated by \citet{24_sparsity} for training and inference where activations are pruned so only 2 of every 4 activations are nonzero. 

\textbf{Lossless memory reduction techniques.}
A few methods have explored lossless activation storage with $\relu$ activations. BitTrain \citep{bittrain} stores the nonzero activation values for training convolutional neural networks (CNNs) on CPU-based edge devices. cDMA \citep{cdma} exploits activation sparsity to reduce transfers between GPU and host memory during training for host activation caching, also for CNNs. Gist \citep{gist} compresses sparse activations with a 8-bit CSR format. BitSparse extends these methods in several ways, including the standard GPU training setup without host synchronization, dedicated fused kernels for low overhead, low precision support and sign-bit packing for higher compression. We view these as essential for adoption in the modern LLM training era. 

In addition, nvCOMP \citep{nvcomp} provides optimized codecs for general and floating data, which can be applied to activations. Activation checkpointing recomputes activations during backpropagation \citep{act_checkpointing}, sacrificing speed for lower memory. \citet{selective_recomputation} showed this to be effective in large LLM training setups with thousands of GPUs. 

\textbf{Activation sparsity for efficient computation.}
The lazy neuron study \citep{lazy_neuron} documents the emergence of sparse $\relu$ activations in trained transformers. Exploiting the sparsity in ReLU networks for inference has long been explored in CNNs \cite{scnn}. More recently, ReLU Strikes Back \citep{relu_strikes_back} investigated $\relu$-based LLMs, exploiting their activation patterns to reduce inference computation. These methods avoid loading weights corresponding to zero activations. Unfortunately, they do not generalize to training, different tokens activate different channels. 
\section{Experiments} \label{sec:experiments}

\subsection{Sparsity in LLMs}
\label{sec:llm-sparsity}
We first examine activation sparsity in the feed-forward layers of two LLMs using $\relu^2$: a 162M-parameter Modded-NanoGPT transformer \citep{modded_nanogpt} and the 8B-parameter Nemotron-H-8B-Base-8K model. These models differ in both size and architecture: Modded-NanoGPT uses transformer blocks, whereas Nemotron-H combines attention, Mamba, and MLP blocks. Full details are given in Appendix~\ref{app:llm-sparsity}.  

The Modded-NanoGPT model is trained following the Track 3 baseline procedure. Activation sparsity is measured on the validation set every 300 steps for every layer. Figure~\ref{fig:train-sparsity} shows an initial sparsity of approximately $\rho=0.5$, rapidly increasing to $\rho=0.82$. This phenomenon has been observed in other works \cite{24_sparsity}.

Next, we evaluate the pretrained Nemotron-H model. The mean activation sparsity is $0.9273$, with minimum batch-wise mean sparsity over $256$ evaluation batches of $0.9176$, corresponding to $92.73\%$ and $91.76\%$ sparsity, respectively. The sparsity is extremely consistent, with the minimum batch-wise mean sparsity only 0.97 percentage points below the mean, meaning activation compression can be used effectively.

\subsection{Benchmarking and Ablations}
BisSparse is benchmarked in Table~\ref{tab:compression-benchmark}, which measures encoding and decoding time for a 1 GiB tensor with $\rho=0.5$ and $\rho=0.9$, with and without sign-bit packing. Decoding reaches 97.4\% of the attainable bandwidth at $\rho=0.5$ without sign-bit packing, while encoding is slower at 60.9\%. BitSparse-sb is slower, since an additional kernel is needed. Higher sparsity reduces encoding time, although effective bandwidth utilization is lower since less data is moved. 

\begin{table}[t]
    \centering
    \caption{Timings for encoding and decoding a $1$~GiB sparse BF16 tensor ($65{,}536\times8{,}192$) with and without sign-bit packing. The bandwidth utilization (Memory BW) includes both read and write. Bandwidth utilization is compared with the attainable throughput measured with NVbandwidth, with the device's nameplate bandwidth ($936$~GB/s) in parentheses.}
    \label{tab:compression-benchmark}
    \small
    \setlength{\tabcolsep}{3pt}
    \begin{tabular}{crrrrrrr}
        \toprule
        Sign-bit & $\rho$ & \multicolumn{2}{c}{Time / ms} & \multicolumn{2}{c}{Rate / GiB/s} & \multicolumn{2}{c}{Memory BW / \%} \\
        packing & & Encode & Decode & Encode & Decode & Encode & Decode \\
        \midrule
        No  & 0.5 & 3.52 & 2.20 & 284.1 & 454.5 & 60.9 (50.9) & 97.4 (81.5) \\
        Yes & 0.5 & 7.03 & 2.79 & 142.2 & 358.4 & 29.9 (25.0) & 75.3 (63.0) \\
        No  & 0.9 & 3.31 & 2.22 & 302.1 & 450.5 & 52.3 (43.8) & 78.0 (65.3) \\
        Yes & 0.9 & 4.84 & 2.92 & 206.6 & 342.5 & 35.4 (29.6) & 58.7 (49.1) \\
        \bottomrule
    \end{tabular}
\end{table}

To assess the benefit of our custom kernels, we compare our Triton implementation with a pure PyTorch implementation of BitSparse on a $65{,}536\times8{,}192$ BF16 tensor. Both versions produce identical outputs. The Triton implementation encodes $7.59\times$ faster and decodes $7.91\times$ faster, while reducing peak allocated memory by $37.5\%$ and $21.9\%$, respectively (Table~\ref{tab:codec-implementation-ablation}, Appendix~\ref{app:codec-implementation-ablation}).

The fused backward gradient kernel is compared with the standard PyTorch version on a $16{,}000\times21{,}504$ BF16 tensor with $50\%$ sparsity. BitSparse reduces peak allocated memory by $46.8\%$ for $\relu$ and $60.9\%$ for $\relu^2$, while reducing runtime by $17.2\%$ and $39.2\%$, respectively. Appendix~\ref{app:relu-backward-ablation} gives the setup, and Table~\ref{tab:relu-backward-ablation} reports the results.

\subsection{Comparisons on training} \label{sec:mlp-training}
\begin{table}[t]
    \centering
    \caption{Peak allocated memory and time per step for the $\relu^2$ MLP benchmark, using BF16 and FP8 activations at sparsity $\rho=0.5$ and $0.9$. CompAct is a lossy reference. Bold and italic values indicate the lowest and second-lowest lossless results. Dashes indicate unavailable results. }
    \label{tab:mlp-training-relu2}
    \small
    \setlength{\tabcolsep}{2.5pt}
    \begin{tabular}{@{}lrrrrrrrr@{}}
        \toprule
        & \multicolumn{4}{c}{Memory / MiB} & \multicolumn{4}{c}{Time / ms} \\
        \cmidrule(lr){2-5}\cmidrule(l){6-9}
        & \multicolumn{2}{c}{BF16} & \multicolumn{2}{c}{FP8} & \multicolumn{2}{c}{BF16} & \multicolumn{2}{c}{FP8} \\
        Method & $0.5$ & $0.9$ & $0.5$ & $0.9$ & $0.5$ & $0.9$ & $0.5$ & $0.9$ \\
        \midrule
        Dense                 & 12346 & 11344 & 9492 & 9492 & \textit{1827} & \textit{1827} & \textit{395} & 395 \\
        BitSparse             & 8806 & \textit{5896} & \textit{7180} & \textit{6617} & \textbf{1822} & \textbf{1815} & \textbf{391} & \textit{391} \\
        BitSparse-sb          & \textit{8637} & \textbf{5860} & \textbf{7042} & \textbf{6551} & 1844 & 1828 & 397 & \textbf{391} \\
        Checkpoint            & \textbf{6713} & 6713 & 7417 & 7417 & 2144 & 2144 & 507 & 507 \\
        LZ4                   & 11086 & 8052 & 8777 & 8242 & 3565 & 2584 & 693 & 615 \\
        Zstd                  & 10064 & 7517 & 8071 & 7664 & 4937 & 3311 & 990 & 784 \\
        Cascaded              & 12207 & 8462 & 9157 & 8921 & 1957 & 2185 & 469 & 461 \\
        Bitcomp               & 12285 & 9154 & 8870 & 8693 & 1926 & 2144 & 461 & 452 \\
        CSR                   & -- & 15859 & -- & -- & -- & 2308 & -- & -- \\
        \midrule
        CompAct-2             & 14111 & 14767 & 10078 & 13284 & 2754 & 2754 & 516 & 516 \\
        CompAct-4             & 12581 & 13237 & 9238 & 12519 & 2037 & 2037 & 399 & 399 \\
        CompAct-8             & 11158 & 12472 & 8854 & 12138 & 1692 & 1692 & 346 & 346 \\
        \bottomrule
    \end{tabular}
\end{table}

\begin{table}[t]
    \centering
    \caption{Effect of batch size on peak allocated memory and time per step for the $\relu^2$ MLP at sparsity $\rho=0.9$. Bold and italic values indicate the lowest and second-lowest results, respectively.}
    \label{tab:mlp-batch-size}
    \small
    \setlength{\tabcolsep}{3pt}
    \begin{tabular}{@{}rrrrrrrrr@{}}
        \toprule
        Batch & \multicolumn{4}{c}{Memory / MiB} & \multicolumn{4}{c}{Time / ms} \\
        \cmidrule(lr){2-5}\cmidrule(l){6-9}
        size & Dense & Checkpoint & BitSparse & BitSparse-sb & Dense & Checkpoint & BitSparse & BitSparse-sb \\
        \midrule
          32 & 3054 & 2879 & \textit{2879} & \textbf{2879} & \textbf{12} & 14 & \textit{12} & 13 \\
         128 & 3094 & 2899 & \textit{2896} & \textbf{2896} & \textbf{18} & 22 & \textit{19} & 19 \\
         512 & 3264 & 2982 & \textit{2972} & \textbf{2971} & \textit{64} & 74 & \textbf{64} & 64 \\
        2000 & 3932 & 3294 & \textit{3252} & \textbf{3249} & 243 & 284 & \textbf{240} & \textit{241} \\
        8000 & 7107 & 4704 & \textit{4384} & \textbf{4368} & 929 & 1086 & \textbf{919} & \textit{926} \\
       32000 & 19811 & 10704 & \textit{8916} & \textbf{8846} & 3640 & 4268 & \textbf{3604} & \textit{3628} \\
        \bottomrule
    \end{tabular}
\end{table}

We benchmark activation storage on a stack of eight residual MLP blocks (Appendix~\ref{app:mlp-training}), with BF16 parameters, BF16 or FP8 activation storage, on a forward and backward pass. Optimizer stats are not included, since neither BitSparse or any baselines affect optimizer buffers. The comparisons cover $\relu$ and $\relu^2$ at sparsity $\rho=0.5$ and $0.9$ in both BF16 and FP8. The \texttt{TensorBuffer} is not used for these experiments so peak memory allocation can be measured.

Comparisons are made against activation checkpointing of each FFN block \cite{act_checkpointing}; lossless codecs as implemented by nvCOMP \cite{nvcomp}, LZ4 \cite{lz4}, Zstd \cite{zstd}, Bitcomp, and Cascaded; compressed sparse row (CSR); and CompAct \cite{CompAct}, a lossy reference. The peak allocated memory and time per step is recorded. Every method except CompAct produces identical gradients up to rounding. 

Table~\ref{tab:mlp-training-relu2} reports the $\relu^2$ results in both BF16 and FP8; Appendix~\ref{app:mlp-additional} gives the corresponding $\relu$ results. At $\rho=0.9$, BitSparse is the most memory-efficient lossless method, dropping VRAM usage by up to 48\% relative to dense storage.
Sign-bit packing reduces memory further, at a small cost in time. For BF16, checkpointing uses the least memory at $\rho=0.5$, but BitSparse overtakes it at $\rho=0.9$. For FP8, BitSparse-sb uses the least memory at both sparsity levels. Checkpointing carries a substantially larger runtime penalty. The other lossless compression methods either have lower compression than BitSparse or run substantially slower. 

To further investigate the memory--runtime trade-off between BitSparse and checkpointing, we vary the number of compressed layers in our model from 0 to 8. Figure~\ref{fig:mlp-layer-pareto} shows the frontier for $\relu^2$ at $\rho=0.9$. Appendix~\ref{app:mlp-layer-pareto} gives the other three configurations. BitSparse dominates the Pareto-frontier, particularly at higher sparsity levels. Also visible is the drop in VRAM on the last data point caused by lower transient memory usage from our fused pointwise gradient kernel. Furthermore, at high sparsity, our $\relu^2$ kernel is faster than the standard implementation. 

We also vary the batch size from $32$ to $32{,}000$ (Table~\ref{tab:mlp-batch-size}). As the batch grows, the activations that BitSparse compresses account for more memory so memory savings increase from $6\%$  at $32$ rows to $55\%$ at $32{,}000$. Runtime stays within $1\%$ of dense storage from $512$ rows upward. BitSparse therefore occupies a different point on the memory–compute frontier from checkpointing: its advantage grows with naturally occurring activation sparsity, while avoiding recomputation.

\begin{figure}[t]
    \centering
    \includegraphics[width=0.5\linewidth]{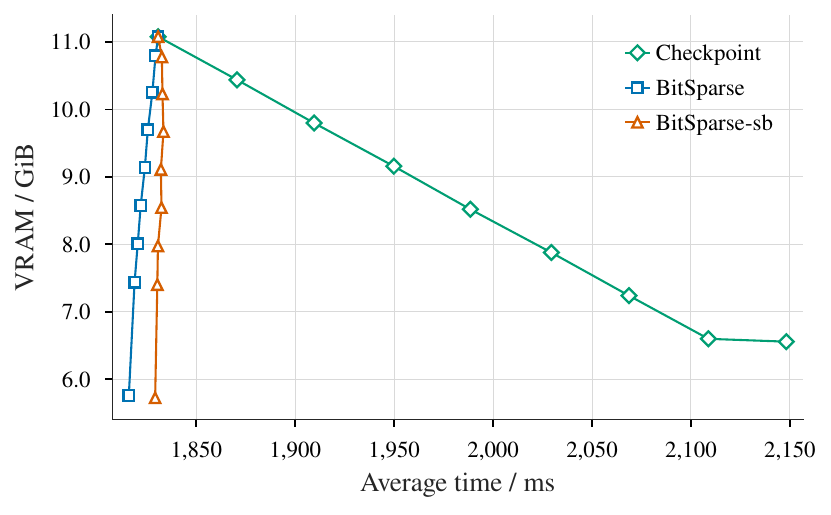}
    \caption{Memory--runtime trade-off for the $\relu^2$ MLP at sparsity $\rho=0.9$, varying the number of compressed or checkpointed layers from $0$ to $8$. Each point represents one layer-count configuration, lines connect configurations in layer-count order.}
    \label{fig:mlp-layer-pareto}
\end{figure}

\subsection{Memory and runtime in LLMs}
\label{sec:llm-memory-runtime}

As a larger training example, we evaluate peak VRAM usage and per-step latency when training Modded-NanoGPT and Nemotron-H across a range of sequence lengths (Appendix~\ref{app:llm-memory-runtime}), comparing dense activation storage, BitSparse(-sb), and activation checkpointing. Again, no optimizer buffer is used. 
A numerical equality check confirms that all methods reproduce the gradients up to floating-point rounding. 

Figures~\ref{fig:llm-peak-memory} and~\ref{fig:llm-step-time} report the memory and runtime results, respectively. Across the measured token counts, VRAM per token on Modded-NanoGPT decreases from 0.402 MiB to 0.277 MiB with BitSparse and 0.276 MiB with BitSparse-sb. On Nemotron-H, the corresponding values are 2.719, 1.894, and 1.890 MiB. This corresponds to an asymptotic VRAM reduction of 31\% on Modded-NanoGPT and 30\% on Nemotron-H. BitSparse incurs a negligible performance penalty. Checkpointing uses less memory but generally incurs a notable runtime cost. Appendix~\ref{app:nemotron_memory} and Table~\ref{tab:nemotron-activation-memory} give a detailed breakdown of the memory used per token for Nemotron-H.

\begin{figure}[t]
    \centering
    \includegraphics[width=0.49\linewidth]{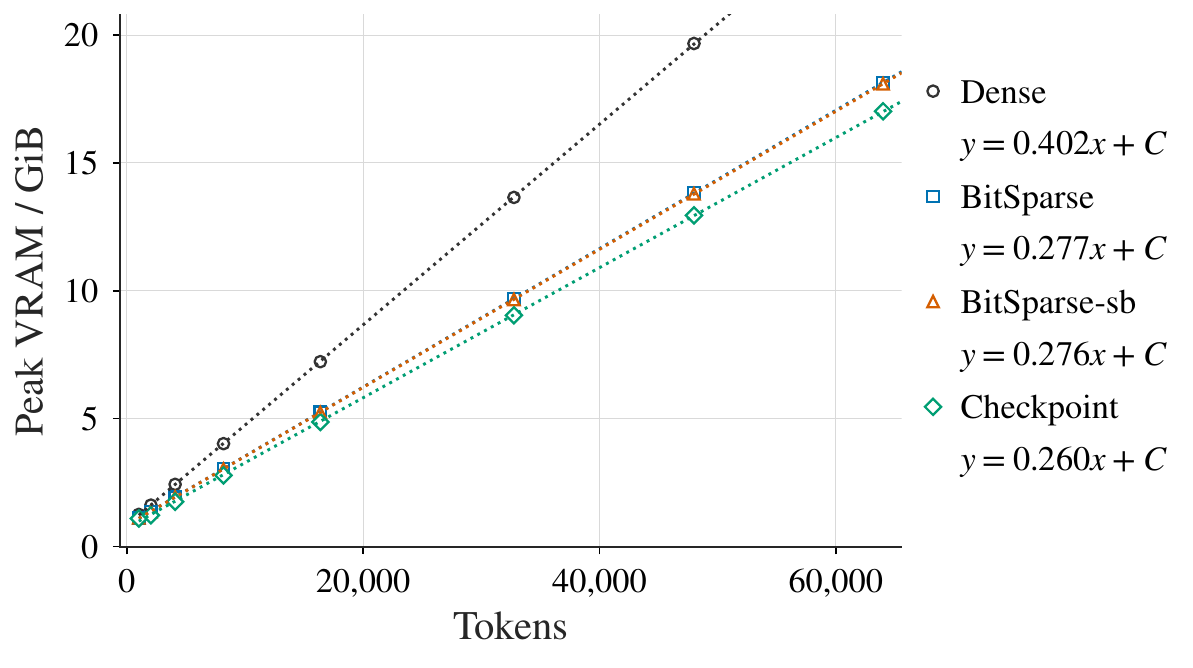}\hfill
    \includegraphics[width=0.49\linewidth]{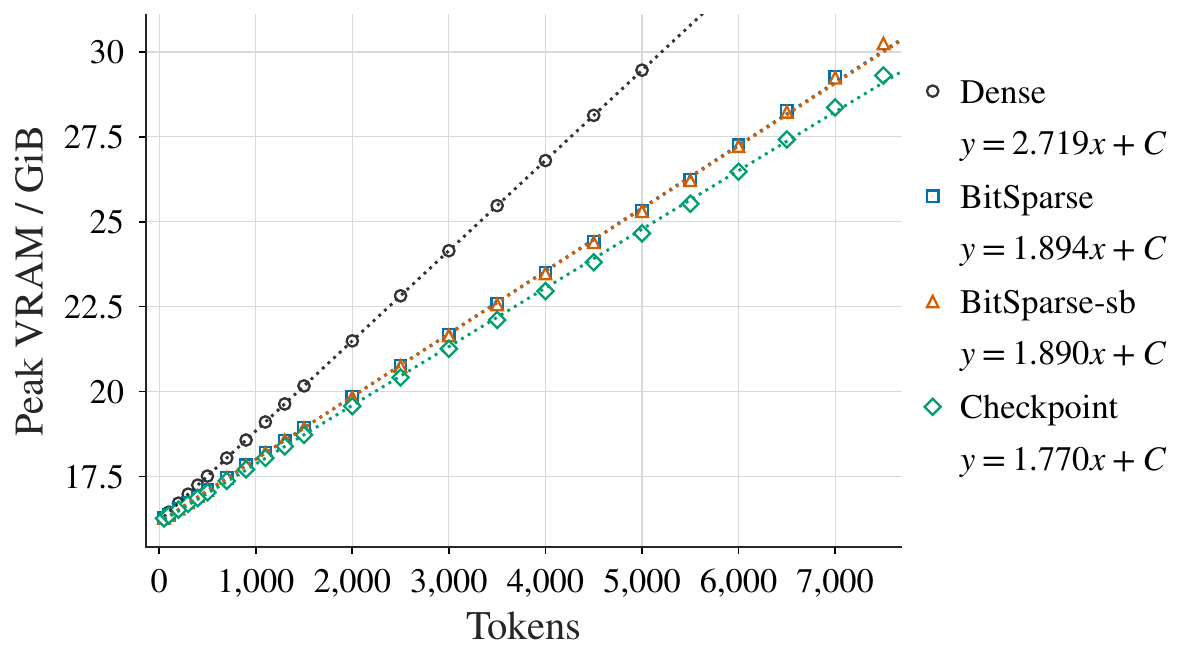}
    \caption{Peak GPU memory for Modded-NanoGPT (left) and Nemotron-H (right) versus the number of tokens processed per step. Markers denote measurements, dotted curves are least-squares linear fits. Legend entries give the fitted slope in units of MiB per token.}
    \label{fig:llm-peak-memory}
\end{figure}

\begin{figure}[t]
    \centering
    \includegraphics[width=0.49\linewidth]{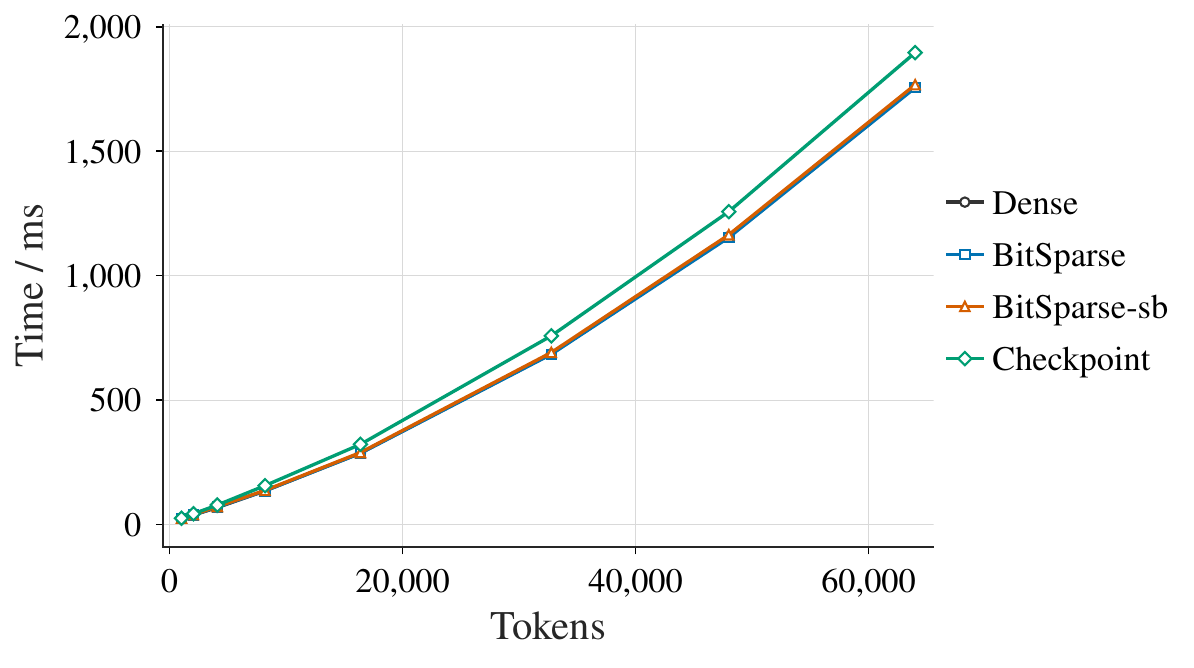}\hfill
    \includegraphics[width=0.49\linewidth]{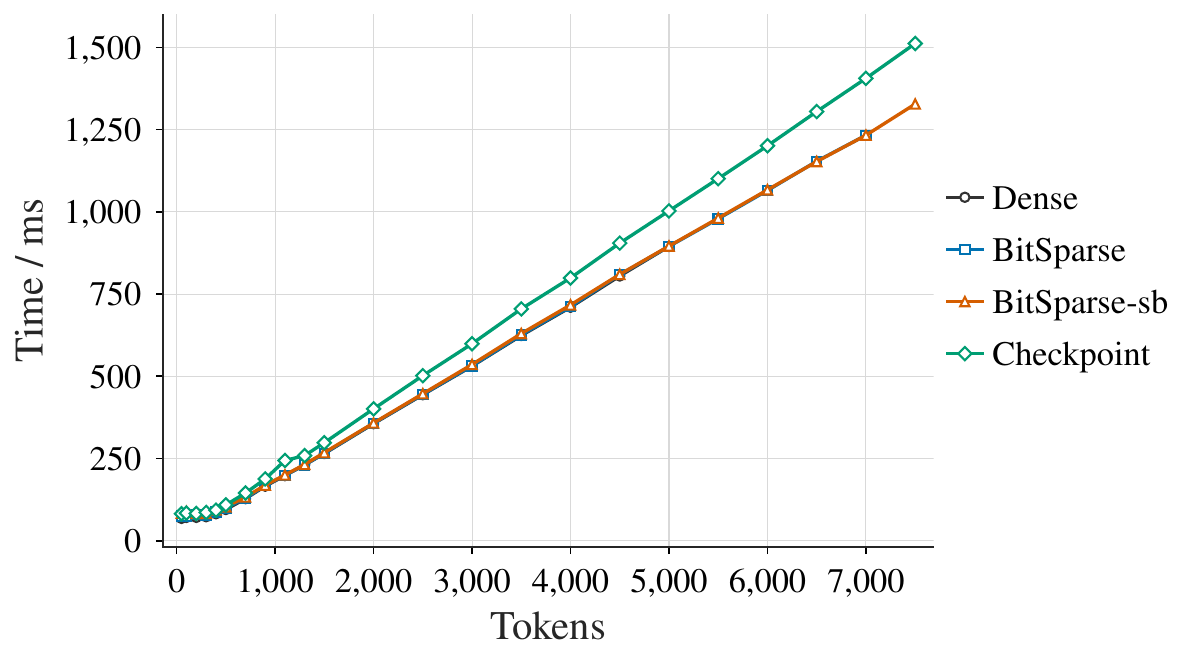}
    \caption{Step time for Modded-NanoGPT (left) and Nemotron-H (right) versus the number of tokens processed per step. Conventions follow Figure~\ref{fig:llm-peak-memory}.}
    \label{fig:llm-step-time}
    \vspace{-5mm}
\end{figure}

\subsection{Full training comparison}
\label{sec:full-training}
To demonstrate the full recipe, Modded-NanoGPT is trained with BitSparse-sb and TensorBuffer over the full $3{,}350$ training steps. Figure~\ref{fig:nanogpt-train-loss} in Appendix~\ref{app:nanogpt-microbatch} shows that the validation loss remains virtually identical to that of the original recipe, ending at $3.2778$ for standard training and $3.2780$ for BitSparse-sb, a difference of $0.0002$, both below the target of $3.28$.

Table~\ref{tab:nanogpt-microbatch} in Appendix~\ref{app:nanogpt-microbatch} compares standard and BitSparse training at microbatch sizes of $32$ and $64$. At size $32$, BitSparse is slightly slower ($7.668$ versus $7.428$ seconds per step). Standard training runs out of memory at size $64$, whereas BitSparse supports the larger microbatch and reduces its step time to $7.416$ seconds, slightly faster than standard training at size $32$.

\section{Conclusion}
\label{sec:conclusion}

We presented \emph{BitSparse}, a GPU-resident tile-wise sparse representation for $\relu$ and $\relu^2$ feed-forward activations that is lossless and can be used as a drop-in replacement to reduce memory usage. It is fast, using optimized fused Triton kernels, with overhead typically less than $1\%$ for large sequence lengths while reducing memory by up to $30\%$. Because this sparsity arises naturally, BitSparse converts existing exact zeros into memory savings without approximation. BitSparse supports both BF16 and FP8 formats and can operate without host synchronization. 

BitSparse depends on the continued success of $\relu$-family activations, since its gains rest on exact structural zeros. The scope is limited to feed-forward activations; attention and Mamba layers are not affected. Future work should extend the representation to gated activations and co-design it with checkpointing, offloading, and distributed training.

\subsection*{AI use statement}
Generative AI assisted with proofreading this paper. The Triton kernels were optimized using AI tools for maximum performance. A simple Autoresearch \cite{autoresearch} prompt was used in combination with manual inspection and direction. Our prompt is included in the supplementary materials. After the kernels were optimized, the rest of the supporting software (including PyTorch autograd integration, experimentation and accuracy validation) was written manually with minor help from AI tooling. Experiments were run manually to ensure accuracy.

\subsection*{Reproducibility statement}
The supplementary materials include the BitSparse implementation (\texttt{lib\_sparse}), benchmark scripts (\texttt{experiments}), and language-model experiments (\texttt{nanogpt} and \texttt{nemotron}).

Section~\ref{sec:methodology} describes the storage format, GPU kernels, and shared buffer; Algorithm~\ref{alg:relu-linear} gives the forward and backward passes. Appendix~\ref{app:proofs} contains the proofs of the entropy and compression results. Appendix~\ref{app:implementation} specifies hardware, software versions, random seeds, and baselines. Appendices~\ref{app:mlp-training} and~\ref{app:llm-memory-runtime} detail the MLP and language-model workloads, precision settings, and timing and memory measurements. Appendix~\ref{app:llm-sparsity} describes the data and aggregation used to measure activation sparsity, and Appendix~\ref{app:additional} provides the full-training comparison, additional benchmark results, and ablations.

\subsection*{Ethics statement}
BitSparse compresses intermediate activations during model training and does not itself collect data or introduce a new data source. Reducing GPU memory requirements may make larger training workloads more accessible. Beyond this, any privacy, bias, and misuse risks depend on the use case, rather than BitSparse. Its environmental impact depends on the memory savings and runtime overhead of the complete training workload.

\bibliography{references}
\bibliographystyle{plainnat}

\clearpage
\appendix
\section{Memory usage during training} \label{app:training-memory}
\paragraph{Reference model and assumptions.}
Here, we describe in more detail the memory usage involved in training LLMs. 

We assume Bfloat16 storage is used everywhere for this analysis (2 bytes per element). Let $P$ be the number of parameters in a model, $W$ be its width, $D$ be its depth and FFN expansion factor $r$. For LLM training, let the sequence length be $n$ and total number of tokens be $N$, equal to sequence length times batch size.

The model parameters require storage of $2P$ bytes. The optimizer states require either a buffer for every parameter, like Muon \cite{muon} and SGDM, or 2 buffers for optimizers such as Adam \cite{adam}. Hence at least $2P$ bytes are needed. If gradients are all stored, an additional $2P$ bytes of storage are needed. However, gradient overhead can be removed by fusing the optimizer into the backward step, where the momentum buffers are updated as soon as gradients are computed, removing the need for gradient storage \cite{lomo}. Hence, in this setup for LLM training or finetuning with Muon, a static memory overhead of $4P$ is required.

The rest of the memory used is for the activations. These scale linearly with input size. While attention complexity is the square of sequence length, the peak memory required is only linear \cite{attention_memory}. Furthermore, Flash attention \cite{flash_attention} provides a highly optimized attention implementation, where the intermediate computations are recomputed in the backward pass rather than stored. This reduces the cache size to $\mathcal{O}(N)$. State-space methods such as Mamba \cite{mamba} also use fused checkpointed kernels for the backward pass for good memory efficiency without saving the full intermediate computation graph. The storage used is $\mathcal{O}(DWN)$.

The feedforward layers make up the other major portion of the activation cache. A feedforward block consists of an up projection, activation and down-projection, $y = \operatorname{activation}(x W_1^\top) W_2^\top$. This projection to the high dimension activation requires a large amount of storage, since its hidden dimension is $rW$. The input to the FFN layer is cached. Since the activation is pointwise, it can be recomputed leading to total memory usage $O(D(r+1)WN)$. 

Finally, we have pointwise and normalization operations, such as RMS and LayerNorm. These pointwise operations can be efficiently checkpointed with fused kernel implementations; hence these layers do not require significant storage. 

\subsection{Nemotron-H memory usage} \label{app:nemotron_memory}

We will use Nemotron-H \citep{nemotron_h} as a reference for size and shapes. 
Its configuration contains $52$ blocks: $24$ MLP, $24$ Mamba, and $4$ attention blocks, with hidden width $4096$, MLP width $21{,}504$ (an expansion factor of 5.25), and vocabulary size $131{,}072$, totaling approximately $8.1$ billion parameters that occupy $15.089$ GiB in Bfloat16. In this paper, we will use BF16 precision for training for simplicity.  

\paragraph{Activations and intermediate tensors.}
Unless stated otherwise, the tensor sizes below are for a micro-batch of $4096$ tokens, so that they can be compared directly with the per-token figures in Table~\ref{tab:nemotron-activation-memory}. A single BF16 residual-width tensor occupies $0.03125$ GiB at this length, whereas an expanded MLP tensor occupies $0.16406$ GiB.
Keeping one expanded tensor for each of the $24$ MLPs costs $3.9375$ GiB. A conventional $\relu^2$ backward graph may retain both the unsquared $\relu$ output and the squared tensor consumed by the down-projection, requiring two such sets ($7.875$ GiB), in addition to the MLP inputs and normalization intermediates. The precise saved set depends on fusion and checkpointing.

Mamba's projected gate, convolution inputs, and time-step values span $18{,}560$ features per token; one projected tensor per Mamba block totals $3.398$ GiB. Convolution outputs, scan intermediates, and normalization can require additional storage, depending on the kernel and chunking strategy. Attention Q, K, and V together span $6144$ features per token, totaling $0.188$ GiB across four blocks. Materializing one BF16 attention-score tensor per block would add $4$ GiB at length $4096$, or $16$ GiB at length $8192$; memory-efficient attention avoids retaining the full score matrices. Finally, the vocabulary projection produces $1$ GiB of BF16 logits at $4096$ tokens, and the implementation casts these to FP32, whose footprint is $2$ GiB. Loss computation and the cast can introduce further temporary allocations.

\begin{table}[t]
    \centering
    \caption{Analytical activation storage per token for Nemotron-H-8B-Base-8K, comparing dense storage, checkpointing of each feed-forward block, BitSparse, and BitSparse with sign-bit packing. Sizes are computed from the model's tensor shapes, not measured, and cover the forward-pass tensors retained for the backward pass. Components below $1$~KiB per token are grouped under \emph{Other}. Values are rounded to the nearest whole KiB, so columns may not sum exactly.}
    \label{tab:nemotron-activation-memory}
    \small
    \setlength{\tabcolsep}{4pt}
    \begin{tabular}{lrrrr}
        \toprule
        & \multicolumn{4}{c}{Storage / (KiB per token)} \\
        \cmidrule(l){2-5}
        Activation & Dense & Checkpoint & BitSparse & BitSparse-sb \\
        \midrule
        Raw inputs retained for block/final norms & 424 & 424 & 424 & 424 \\
        Mamba packed projection (\texttt{gate}, \texttt{x}, \texttt{B}, \texttt{C}, \texttt{dt}) & 870 & 870 & 870 & 870 \\
        Mamba scan output & 384 & 384 & 384 & 384 \\
        FlashAttention Q/K/V & 48 & 48 & 48 & 48 \\
        FlashAttention output & 32 & 32 & 32 & 32 \\
        FFN unsquared $\relu$ values & 1008 & 0 & 81 & 76 \\
        FFN masks, prefixes and offsets & 0 & 0 & 63 & 63 \\
        \emph{Other} (RMS and attention statistics, metadata) & 2 & 1 & 2 & 2 \\
        \midrule
        Total & 2768 & 1759 & 1905 & 1899 \\
        \bottomrule
    \end{tabular}
\end{table}

\paragraph{Implementation optimizations.}
Two lossless changes to the Nemotron-H implementation keep its measured footprint close to the model's intrinsic memory use. The RMS normalization is fused into a single kernel, so its intermediate statistics and normalized values are recomputed during the backward pass rather than retained. The vocabulary projection head is evaluated in chunks along the vocabulary dimension, so the full FP32 logits tensor described above is never materialized at once and only the chunk being processed is live. Neither change alters the computation performed, only where its intermediates are stored, so the results match the unfused formulation up to floating-point rounding. Both are applied to the dense baseline as well as to BitSparse, so the comparisons in Section~\ref{sec:llm-memory-runtime} reflect the choice of activation storage rather than incidental implementation differences.

\paragraph{Effect of activation compression.}
Using the mean observed sparsity $\rho=0.9273$ as an illustrative cache sparsity, Equation~\eqref{eq:bitsparse-asymptotic} reduces one set of expanded BF16 MLP activations from $3.9375$ GiB to approximately $0.534$ GiB without sign-bit packing or $0.516$ GiB with packing. These estimates include $0.246$ GiB of masks and $0.00192$ GiB of tile-prefix metadata. Using the minimum batch-wise mean sparsity $\rho=0.9176$ instead gives $0.572$ and $0.552$ GiB, respectively. The sparsity measurements were taken after squaring, while the cache stores unsquared $\relu$ values, so these are estimates rather than exact cache measurements. Reconstruction can still require a dense temporary, and the attention, Mamba, residual, and output-head tensors are not compressed by this method.

\paragraph{Scaling and relation to the benchmark.}
At fixed precision, parameter-related storage is independent of input length and microbatch size. The linear-width activation footprints above scale by $M/4096$, where $M$ counts all tokens in the microbatch; materialized attention scores instead scale with batch size times squared sequence length. Wider dense projections increase parameter counts approximately quadratically, while the listed feature tensors grow linearly with width when expansion ratios and head dimensions are fixed. The measurements in Appendix~\ref{app:llm-memory-runtime} omit optimizer updates and release parameter gradients as they are produced. They therefore exclude optimizer state and a retained full-model gradient buffer, but transient gradients and operator workspaces still contribute to the measured peak. Inference KV and recurrent-state caches are disabled for these forward/backward runs. Since tensor lifetimes overlap only partially and allocator reservation may exceed live storage, the rows above should not be summed into a predicted peak-memory total.

\section{Analysis of storage usage} \label{app:proofs}

\paragraph{Binary entropy function.}
All logarithms are base two. The binary entropy function is
\begin{equation}
    h_2(\rho)=-\rho\log_2\rho-(1-\rho)\log_2(1-\rho),
    \label{eq:binary-entropy}
\end{equation}
with the convention $0\log_2 0=0$ at the endpoints.

\begin{proof}[Proof of Theorem~\ref{thm:sparse-entropy}]
For an abstract entry $h$, let $S$ indicate whether it is a payload symbol rather than the distinguished zero symbol. Since $S$ is determined by $h$, the entropy chain rule gives
\[
    \mathcal{H}(h)
    =\mathcal{H}(S)+\mathcal{H}(h\mid S)
    =h_2(\rho)+(1-\rho)(d-1),
\]
because $h_2(1-\rho)=h_2(\rho)$, the zero symbol is deterministic and each payload has $d-1$ independent uniform bits. Independence of the entries makes their entropies additive, yielding the result.
\end{proof}

\begin{lemma}[The uniform-bit model is an upper bound]
\label{lem:entropy-upper-bound}
Let $X$ be an actual activation drawn from an arbitrary distribution over the finite nonnegative values representable in a $d$-bit format, with $\Pr(X=0)=\rho$, and let $\mathcal{V}$ denote the set of representable finite positive values. Then
\begin{equation}
    \mathcal{H}(X)
    \le h_2(\rho)+(1-\rho)\log_2\lvert\mathcal{V}\rvert
    \le h_2(\rho)+(1-\rho)(d-1).
    \label{eq:entropy-upper}
\end{equation}

Moreover, for any joint distribution $(X_1,\dots,X_L)$ whose entries share these marginals,
\begin{equation}
    \mathcal{H}(X_1,\dots,X_L)
    \le L\bigl[h_2(\rho)+(1-\rho)(d-1)\bigr],
    \label{eq:entropy-joint-upper}
\end{equation}
which is the value in Theorem~\ref{thm:sparse-entropy}.
\end{lemma}
\begin{proof}
Let $S=\mathbb{I}[X>0]$, so that $\Pr(S=0)=\rho$ and $\mathcal{H}(S)=h_2(\rho)$. Since $S$ is determined by $X$, the chain rule gives
\[
    \mathcal{H}(X)=\mathcal{H}(S)+\mathcal{H}(X\mid S)
    =h_2(\rho)+(1-\rho)\,\mathcal{H}(X\mid S=1),
\]
where the conditional term retains the factor $1-\rho$ because $X$ is identically zero when $S=0$, a deterministic event. Conditional on $S=1$, the value $X$ lies in $\mathcal{V}$, and the entropy of any distribution supported on a finite set is at most the logarithm of its cardinality, so $\mathcal{H}(X\mid S=1)\le\log_2\lvert\mathcal{V}\rvert$. A $d$-bit format admits $2^{d}$ bit patterns, exactly $2^{d-1}$ of which carry a zero sign bit, and every positive value is one of those, so $\lvert\mathcal{V}\rvert\le 2^{d-1}$ and $\log_2\lvert\mathcal{V}\rvert\le d-1$. This gives Equation~\eqref{eq:entropy-upper}.

Entropy is subadditive, so
\[
    \mathcal{H}(X_1,\dots,X_L)\le\sum_{\ell=1}^{L}\mathcal{H}(X_\ell)
    \le L\bigl[h_2(\rho)+(1-\rho)(d-1)\bigr],
\]
which is Equation~\eqref{eq:entropy-joint-upper}. Subadditivity is an equality if and only if the entries are independent. The per-entry bound using $\log_2\lvert\mathcal{V}\rvert$ is attained when the positive values are uniform over $\mathcal{V}$; the looser bound using $d-1$ is strict for actual floating-point values.
\end{proof}

\paragraph{Remarks.}
Lemma~\ref{lem:entropy-upper-bound} separates the three relaxations that Equation~\eqref{eq:entropy-upper} permits. The support need not be full: $\lvert\mathcal{V}\rvert$ may be smaller than $2^{d-1}$ if some bit patterns are never produced. The positive values need not be uniform. Individual entries need not be independent, and subadditivity charges no penalty for dependence. Each relaxation can only lower the entropy, so the uniform, independent model of Assumption~\ref{ass:activation_distribution} is the worst case for compression and therefore a valid reference rate.

The bound is in fact strict. Zero is itself one of the $2^{d-1}$ patterns with a zero sign bit, so $\lvert\mathcal{V}\rvert\le 2^{d-1}-1$ and $\log_2\lvert\mathcal{V}\rvert<d-1$; formats that reserve patterns for infinities or NaN have still fewer finite positive values. The reference entropy in Theorem~\ref{thm:sparse-entropy} therefore exceeds the actual tensor entropy by at least $L(1-\rho)\bigl[(d-1)-\log_2\lvert\mathcal{V}\rvert\bigr]$. This gap is exact when entries are independent and their positive values are uniform over $\mathcal{V}$; nonuniformity or dependence can increase it. The corresponding per-entry gap omits the factor $L$. We use the reference bound because it depends only on the format width.

\begin{proof}[Proof of Theorem~\ref{thm:bitsparse-compression}]
The representation in Section~\ref{sec:method-tile-format} stores $s=(1-\rho)L$ positive values, a $TB^2$-bit mask, and $32T$ bits of prefix metadata. Writing $c=d$ without packing and $c=d-1$ with packing, its occupied storage is exactly
\begin{equation}
    S_{\mathrm{BS}}=cs+TB^2+32T+\Delta,
    \qquad
    R_{\mathrm{BS}}=\frac{dL}{S_{\mathrm{BS}}},
    \label{eq:bitsparse-storage}
\end{equation}
where $\Delta$ includes alignment, guard storage, and per-tensor metadata, including an optional terminal prefix entry. With padding factor $\alpha=TB^2/L$, the cost per logical entry is
\begin{equation}
    \frac{S_{\mathrm{BS}}}{L}
    =c(1-\rho)+\alpha\left(1+\frac{32}{B^2}\right)+\frac{\Delta}{L}.
    \label{eq:bitsparse-bits-per-element}
\end{equation}
For negligible padding and per-tensor overhead, $\alpha\simeq1$ and $\Delta/L\simeq0$. Substituting $B=64$ gives a metadata cost of $1+32/4096=1+1/128$ bits per entry. Dividing the dense cost $d$ by the resulting cost per entry, for each choice of $c$, gives Equation~\eqref{eq:bitsparse-asymptotic}.
\end{proof}

\section{Fused activation--linear layer} \label{app:algorithm}

\subsection{\texorpdfstring{$\relu$}{ReLU} gradient derivation}
\label{app:relu-gradients}
Using the notation of Section~\ref{sec:relu_networks}, the forward pass is
\begin{equation}
    H=\relu(Z),
    \qquad
    Y=HW^\top.
\end{equation}
Differentiating the loss gives
\begin{equation} \label{eq:relu-backward}
    \frac{\partial\mathcal{L}}{\partial W} = G^\top H,
    \qquad
    \frac{\partial\mathcal{L}}{\partial H} = GW,
    \qquad
    \frac{\partial\mathcal{L}}{\partial Z}
    =\frac{\partial\mathcal{L}}{\partial H}\odot \relu'(Z)
    =\frac{\partial\mathcal{L}}{\partial H}\odot\mathbb{I}[Z>0],
\end{equation}
where we define $\relu'(0)=0$. The gradients with respect to $Z$ and $W$ depend only on the positive values of $H$ and their positions, which are cached from the forward pass.

\subsection{Forward and backward algorithm}
Algorithm~\ref{alg:relu-linear} sets out the forward and backward passes of the fused
$\relu$--linear autograd function. The forward pass evaluates the activation in place,
compresses it, and then performs a dense matrix multiplication, so the forward computation
is unchanged from the dense baseline. Only the cached activation is compressed. The backward pass decompresses the activation and computes the weight gradient directly, and applies the
activation derivative from the stored bitmask alone, so the dense activation is never
reconstructed for the input gradient.
\begin{algorithm}[H]
    \caption{BitSparse fused $\relu$--linear layer.}
    \label{alg:relu-linear}
    \begin{algorithmic}[1]
        \Require preactivation $Z \in \mathbb{R}^{M \times N}$, weight $W \in \mathbb{R}^{D \times N}$
        \Ensure output $Y \in \mathbb{R}^{M \times D}$, cache holding the compressed activation
        \Procedure{Forward}{$Z, W$}
            \State $H \gets \relu(Z)$ \Comment{in place}
            \State cache $W$
            \State $H_{\mathrm{s}} \gets \Call{Compress}{H}$
            \Comment{per-tile bitmask, compact positive values, and prefix offsets; optional sign-bit packing}
            \State $Y \gets H W^\top$ \Comment{dense matrix multiplication}
        \EndProcedure
        \Statex
        \Procedure{Backward}{$G$} \Comment{$G = \partial\mathcal{L} / \partial Y \in \mathbb{R}^{M \times D}$, the upstream gradient}
            \State retrieve $W$ and $H_{\mathrm{s}}$ from the cache
            \State $\partial W \gets \Call{SparseProduct}{G^\top, H_{\mathrm{s}}}$
            \Comment{dense--sparse product against the compressed activation}
            \If{$\partial Z$ is required}
                \State $\partial H \gets G W$ \Comment{dense matrix multiplication}
                \State $\partial Z \gets \Call{MaskWithBitmask}{\partial H, \operatorname{bitmask}(H_{\mathrm{s}})}$
                \Comment{in place, applying $\relu'(Z)$ from the bitmask alone}
            \Else
                \State $\partial Z \gets \varnothing$
            \EndIf
            \State release $H_{\mathrm{s}}$
        \EndProcedure
    \end{algorithmic}
\end{algorithm}
The stored values are BF16 or FP8. For FP8 the per-tensor scale is carried in the sparse
metadata and applied to the matrix products; sign-bit packing is an independent option that
removes the redundant sign bit, as described in Section~\ref{sec:method-packing}. The
$\relu^2$ case follows the a similar structure, with the stored values squared when reconstruction is required and \texttt{MaskWithBitmask} mask replaced with a masked multiplication.

\section{Experimental setup} \label{app:implementation}
\subsection{System setup} 
All experiments were implemented using PyTorch 2.14 with Triton 3.7.1 in Python 3.14. All experiments were run on an NVIDIA RTX 3090 24GB with the exception of the FP8 MLP and Nemotron-H experiments, which were run on an RTX 5090 32GB, for the FP8 support and extra VRAM. CompAct used the official implementation. The nvCOMP baseline uses version 5.3.0.16. 

Random seed 0 is used for all experiments. Timings are averaged over 5 runs, after warmup. 

\subsection{Baselines} \label{app:baselines}
\paragraph{Activation Checkpointing.} Activation checkpointing \cite{act_checkpointing} reduces the amount of memory required by recomputing activations in the backward pass, rather than storing them in the forward pass. This trades off storage for computation, since certain layers will have computations repeated twice in a training step. PyTorch supports activation checkpointing natively. In our method, we checkpoint the full feedforward block, including the RMS norm layer.

\paragraph{Lossless compression.} Lossless compression can be used to reduce the size of the saved activations. The activations are compressed before being cached, and decompressed in the backward pass when needed. The nvCOMP \cite{nvcomp} library provides highly optimized GPU-native codecs for this task, including LZ4 \cite{lz4}, Zstd \cite{zstd}, Bitcomp, and Cascaded. LZ4 uses dictionary matching to replace repeated byte sequences with short references. Zstd similarly exploits repeated patterns, while also using entropy coding to improve compression ratio. Bitcomp is a proprietary format designed for numerical data and can efficiently encode sparse values. Cascaded combines techniques such as run-length encoding, delta encoding, and bit-packing. While these are general lossless compression codecs, they benefit greatly from activation sparsity, since the many zeros are easy to compress.

\paragraph{Compressed Sparse Row.} A standard format for sparse matrices is the CSR format. This stores the nonzero values, their column indices and a cumulative counter of row indices, and is well optimized for GPU multiplication, with native support on our GPUs. This method is however suboptimal for this setting, where sparsity is around $0.5$ to $0.85$, where the overhead from indices is significant. Currently, there is no FP8 support for CSR tensors. 

\paragraph{CompAct.} Lossy compression has been considered. We compare against a recent method CompAct \cite{CompAct}, which uses random Gaussian matrices to compress the hidden pre-nonlinearity activations into a lower dimension space. This method exploits the low-rank nature of activations. Training using CompAct demonstrates lower memory, with a small increase in training loss due to its lossy nature. The amount of compression can be tuned. We use CompAct-$f$ to denote a compression factor of $f$, meaning its compression dimension is $1/f$ times the activation. A higher $f$ gives lower memory and faster compression, at the cost of lower accuracy.

\subsection{MLP training benchmark setup} \label{app:mlp-training}

\paragraph{Architecture and workload.}
The comparisons in Section~\ref{sec:mlp-training} use a stack of eight residual MLP blocks. Each block applies RMS normalization, a bias-free linear expansion from $4096$ to $21{,}504$ features (the same as Nemotron-H), a $\relu$ or $\relu^2$ activation, and a bias-free projection back to $4096$ features before adding the residual connection. Inputs are synthetic standard-normal batches with $16{,}000$ rows and $4096$ features. Parameters and activations use BF16.

In this test, parameters are not updated since we are interested in the performance and memory for computing the gradients, not the optimizer. Each lossless method is validated by checking that its gradients match the dense baseline up to floating-point rounding.

\paragraph{FP8 activation storage and matrix multiplication} \label{sec:method-fp8}
For the FP8 experiment we use FP8 only for the matrix multiplications and for the activation cache; the model's parameters, inputs, and gradients remain in BF16. In each product the BF16 operand is cast to FP8 on the fly, so both the activation and the weight are represented in FP8-E4M3 for the GEMM, which accumulates in FP32 and returns a BF16 result. The activation cache stores the FP8 values themselves, together with the scale needed to interpret them. The same FP8 scaling is used as in Nemotron-H, with a tensor-wise scale $\gamma$:
\begin{equation}
    \gamma = \max\left(
        \frac{\max_{i,j}|H_{ij}|}{q_{\max}},
        10^{-9}
    \right),
    \qquad
    H_{\mathrm{FP8}} =
    \operatorname{cast}_{\mathrm{FP8}}\left(\frac{H}{\gamma}\right),
    \label{eq:fp8-scale}
\end{equation}
where $q_{\max}$ is the largest finite value of the selected FP8 format, so $\gamma$ places the largest magnitude at the top of the FP8 range. We use a separate tensor-wise scale for each operand, the weight and the activation. For the $\relu^2$ path, conversion is applied to the saved $\relu(Z)$ values. The implementation currently converts tensors to E4M3 FP8. The activation scale is stored alongside the sparse cache so that the values can be interpreted and scaled back during reconstruction. The weight scale is transient, recomputed each time the weight is cast, so keeping weights in BF16 and casting them per product adds no persistent parameter storage.
Dense FP8 products use PyTorch's scaled matrix multiplication primitive with tensor-wise scales. This scaling ensures the full dynamic range of FP8 values is used. Our method is also compatible with more advanced formats such as MXFP8 which use blockwise multiplication \cite{mx_fp}. 

\paragraph{Setup details.}
All methods use the same weights and inputs, fixed using a random seed. In order to generate sparse activations, the weights of the networks are shifted until the activations have the desired sparsity level. Every method uses a custom backward implementation which fuses the RMS norm layer via $\texttt{torch.autograd.Function}$. This is to ensure maximum efficiency since we found the default PyTorch implementation to be inefficient in its activation caching. In particular, fusing the RMS norm layer reduced the memory consumption of every method. 

\paragraph{Timing and memory.}
Each configuration is measured five times. Each measurement excludes one warmup forward/backward pass and is followed by three timed steps that are averaged for the displayed time. GPU execution is synchronized at the timing boundaries. We record peak GPU memory allocated during the final measured pass, including model parameters, activations, and temporary tensors, rather than only the compressed activation footprint.

\subsection{Language-model memory and runtime setup}  \label{app:llm-memory-runtime}
\paragraph{Models and data.}
We evaluate two models. The first is the 162M-parameter Modded-NanoGPT checkpoint from Section~\ref{sec:llm-sparsity}, trained by us according to the Track 3 baseline and evaluated at step $3300$ ($12$ layers, $768$ hidden dimensions, $50{,}304$-token vocabulary). The second is NVIDIA's official pretrained Nemotron-H-8B-Base-8K model. Both run in BF16 on a single GPU: an RTX 3090 with $24$ GiB for Modded-NanoGPT and an RTX 5090 with $32$ GiB for Nemotron-H. Models are evaluated on text from the FineWeb10B validation dataset \cite{fineweb}, truncated to the target length.
\paragraph{Batch size.}
Modded-NanoGPT is evaluated with a fixed batch of $4$ sequences, sweeping the sequence length from $256$ to $16{,}000$ tokens, so the number of tokens per step ranges from $1{,}024$ to $64{,}000$. Nemotron-H is evaluated one sequence at a time at each length.

\paragraph{Configurations.}
Each model is measured with dense activation storage, activation checkpointing on every feed-forward block, BitSparse, and BitSparse with sign-bit packing. Both models use $\relu^2$ feed-forward activations. BitSparse compresses each activation at the sparsity the model exhibits naturally (Section~\ref{sec:llm-sparsity}). We modify both models to use the same fused RMS normalization, and chunk the Nemotron-H vocabulary projection head, to keep incidental memory overhead low (Appendix~\ref{app:nemotron_memory}). 

\paragraph{Workload.}
Each step performs a forward pass and the backward pass of the language-model loss over that step's tokens. Model parameters are never updated and gradients are released as they are produced, so no full-model parameter-gradient buffer or optimizer state is retained. Transient parameter gradients still contribute to peak memory.

\paragraph{Timing and memory.}
The timing and peak memory consumption in a training step are measured. Times are averaged per step over repeated iterations after warmup. Memory is tracked using \nolinkurl{torch.cuda.max_memory_allocated("cuda")}. As in Appendix~\ref{app:mlp-training}, peak memory covers model parameters, activations, and temporary tensors. Warmup steps are used to ensure accurate timing.

\section{Additional results}
\label{app:additional}

\subsection{Activation sparsity measurement in LLMs} \label{app:llm-sparsity}

The experiments in Section~\ref{sec:llm-sparsity} characterize feed-forward activation sparsity during Modded-NanoGPT training and in a pretrained Nemotron-H model. We evaluate both models on contiguous sequences from the FineWeb10B validation set, with fixed weights and without BitSparse compression. Both experiments use $256$ evaluation batches, but the batch composition differs: the Modded-NanoGPT runs place four $4096$-token sequences in each batch, while the Nemotron-H runs process a single sequence per batch.

\paragraph{Measurement and aggregation.}
Following Section~\ref{sec:background-entropy}, we report activation sparsity $\rho$; larger values indicate a greater fraction of zeros. For each layer $\ell$ and evaluation batch $j$, we measure $\rho_{\ell,j}$ over all tokens and feature dimensions of the $\relu^2$ output, counting exact zeros without a magnitude threshold. With $Q$ evaluation batches and measured layers $\mathcal{I}$, we compute
\begin{align}
    \overline{\rho}_{\ell}
        &= \frac{1}{Q}\sum_{j=1}^{Q}\rho_{\ell,j},
    &\rho_{\ell,\min}
        &= \min_{1\leq j\leq Q}\rho_{\ell,j}, \\
    \rho_{\mathrm{batch},\min}
        &= \min_{1\leq j\leq Q}
           \frac{1}{|\mathcal{I}|}\sum_{\ell\in\mathcal{I}}\rho_{\ell,j}.
    \label{eq:sparsity-batch-aggregation}
\end{align}
These are the per-layer mean sparsity, the per-layer minimum sparsity, and the minimum batch-wise mean sparsity across layers, respectively. The last quantity captures the densest observed batch averaged over layers; individual layers may reach their minima on different batches. Layers receive equal weight. All minima refer to the evaluated batches, rather than bounds on unseen inputs. Measurements are taken after squaring, so finite-precision underflow may produce additional zeros relative to the unsquared activations saved by BitSparse.

\paragraph{Modded-NanoGPT during training.}
We evaluate all $12$ feed-forward layers at training steps $0,300,\ldots,3300$. Each checkpoint is evaluated on the same $256$ batches, each containing four sequences of $4096$ tokens, for a total of $1{,}024$ sequences and $4{,}194{,}304$ input-token positions. Four preliminary warmup batches are excluded. The model uses BF16 activations. We report each layer's mean and minimum activation sparsity to show how sparsity develops during training and how much it varies across batches. In the training-evolution plot, solid and dashed lines denote the mean and minimum, respectively, with a shared color for each layer.

\begin{figure}[t]
    \centering
    \includegraphics[width=0.8\linewidth]{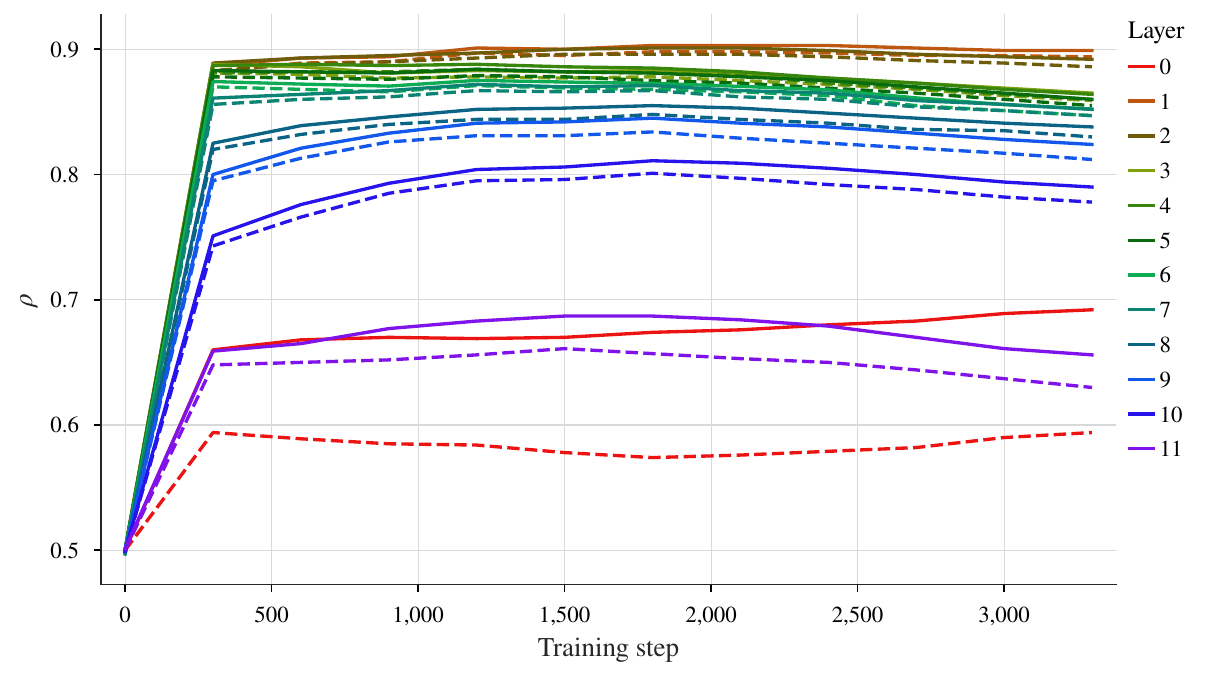}
    \caption{Activation sparsity during Modded-NanoGPT training for all 12 feed-forward layers. Solid lines show per-layer mean sparsity and dashed lines show per-layer minimum sparsity over 256 evaluation batches. The vertical axis reports the zero fraction $\rho$, so higher values indicate greater sparsity.}
    \label{fig:train-sparsity}
\end{figure}

\paragraph{Pretrained Nemotron-H.}
We evaluate the pretrained Nemotron-H model in BF16. Measurements cover all feed-forward MLP blocks in the hybrid architecture, excluding attention and Mamba blocks; layer indices refer to positions in the full model. Evaluation comprises $256$ sequences of $4096$ tokens, processed one sequence per batch, for $1{,}048{,}576$ input-token positions. One warmup pass is excluded from the statistics. We report per-layer mean and minimum sparsities, together with the mean sparsity across layers, the smallest individual layer--batch sparsity, and the minimum batch-wise mean sparsity defined above.

\begin{table}[t]
    \centering
    \caption{Activation sparsity $\rho$ for selected Nemotron-H-8B-Base-8K layers over 256 validation batches. Per-layer minima correspond to each layer's densest observed batch. The overall row reports mean sparsity across all measured layers and batches, and the minimum batch-wise mean sparsity across all measured layers. Higher values indicate greater sparsity.}
    \label{tab:nemotron-sparsity}
    \small
    \begin{tabular}{rrr}
        \toprule
        Layer & Mean sparsity & Minimum sparsity \\
        \midrule
         1 & 0.9718 & 0.9673 \\
         3 & 0.9335 & 0.9235 \\
        10 & 0.9590 & 0.9543 \\
        27 & 0.8649 & 0.8408 \\
        32 & 0.9160 & 0.9023 \\
        47 & 0.9616 & 0.9383 \\
        51 & 0.9357 & 0.8789 \\
        \midrule
        Overall & 0.9273 & 0.9176 \\
        \bottomrule
    \end{tabular}
\end{table}

\FloatBarrier
\Needspace{0.55\textheight}
\subsection{Full training comparison}
\label{app:nanogpt-microbatch}
Table~\ref{tab:nanogpt-microbatch} compares the memory use and step time of standard and BitSparse training at microbatch sizes of $32$ and $64$. Figure~\ref{fig:nanogpt-train-loss} shows the validation loss over the full training run.

\begin{table}[htbp]
    \centering
    \caption{Memory usage and time per training step for Modded-NanoGPT at two microbatch sizes. OOM denotes an out-of-memory failure.}
    \label{tab:nanogpt-microbatch}
    \small
    \begin{tabular}{@{}lrrr@{}}
        \toprule
        Method & Microbatch size & Memory / MiB & Time / s \\
        \midrule
        Normal & 32 & 15229.6 & 7.428 \\
        Normal & 64 & OOM & OOM \\
        \midrule
        BitSparse & 32 & 13174.7 & 7.668 \\
        BitSparse & 64 & 21822.1 & 7.416 \\
        \bottomrule
    \end{tabular}
\end{table}

\begin{figure}[htbp]
    \centering
    \includegraphics[width=0.75\linewidth]{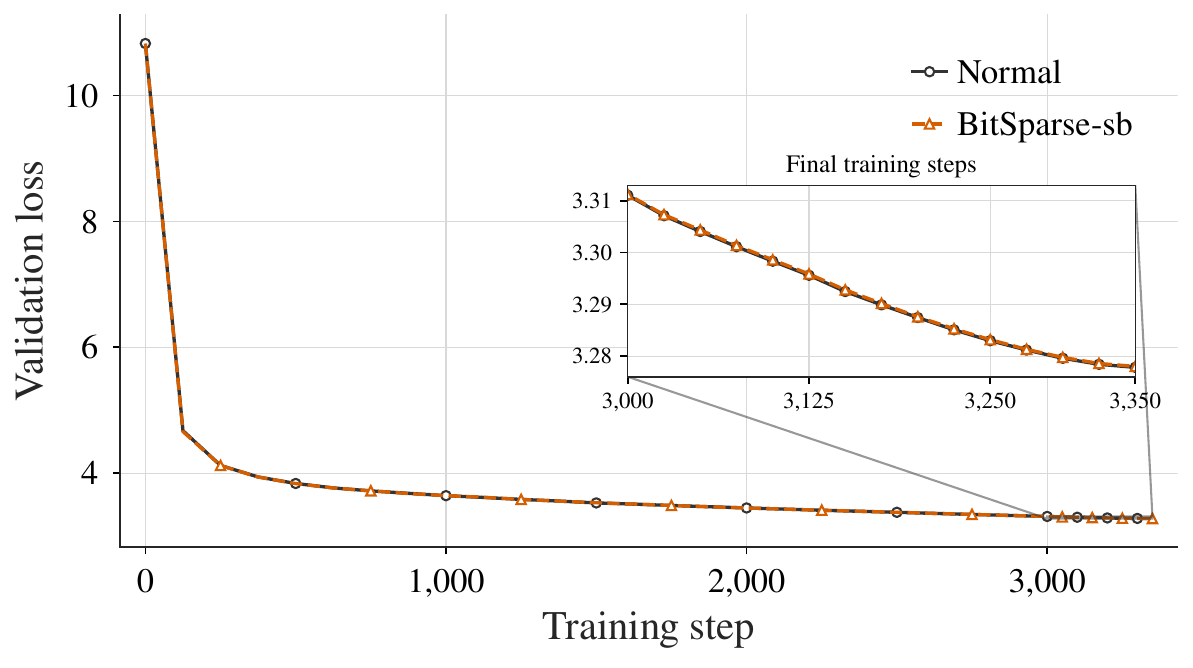}
    \caption{Validation loss during standard (Normal) and BitSparse-sb training of Modded-NanoGPT. The inset zooms in on steps $3{,}000$--$3{,}350$.}
    \label{fig:nanogpt-train-loss}
\end{figure}

\FloatBarrier
\subsection{MLP benchmark results} \label{app:mlp-additional}

Table~\ref{tab:mlp-training-relu} reports the $\relu$ MLP benchmark at sparsity $\rho=0.5$ and $0.9$ for both BF16 and FP8. Table~\ref{tab:mlp-training-relu2} gives the corresponding $\relu^2$ results discussed in the main text.

\begin{table}[htbp]
    \centering
    \caption{Peak allocated memory and time per step for the $\relu$ MLP benchmark, using BF16 and FP8 activations at sparsity $\rho=0.5$ and $0.9$. CompAct is a lossy reference. Bold and italic values indicate the lowest and second-lowest lossless results. Dashes indicate unavailable results.}
    \label{tab:mlp-training-relu}
    \small
    \setlength{\tabcolsep}{2.5pt}
    \begin{tabular}{@{}lrrrrrrrr@{}}
        \toprule
        & \multicolumn{4}{c}{Memory / MiB} & \multicolumn{4}{c}{Time / ms} \\
        \cmidrule(lr){2-5}\cmidrule(l){6-9}
        & \multicolumn{2}{c}{BF16} & \multicolumn{2}{c}{FP8} & \multicolumn{2}{c}{BF16} & \multicolumn{2}{c}{FP8} \\
        Method & $0.5$ & $0.9$ & $0.5$ & $0.9$ & $0.5$ & $0.9$ & $0.5$ & $0.9$ \\
        \midrule
        Dense                 & 11046 & 10039 & 8500 & 8500 & \textbf{1748} & \textbf{1748} & \textbf{323} & 333 \\
        BitSparse             & 8800 & \textit{5742} & 7165 & 6495 & \textit{1794} & \textit{1776} & \textit{324} & \textbf{328} \\
        BitSparse-sb          & \textit{8631} & \textbf{5715} & \textit{7030} & \textit{6443} & 1799 & 1786 & 330 & \textit{331} \\
        Checkpoint            & \textbf{6928} & 6057 & \textbf{6105} & \textbf{6105} & 2047 & 2047 & 441 & 452 \\
        LZ4                   & 10721 & 7324 & 8428 & 7700 & 3334 & 2384 & 644 & 544 \\
        Zstd                  & 9707 & 6942 & 7727 & 7188 & 4702 & 3046 & 934 & 691 \\
        Cascaded              & 11827 & 7533 & 8825 & 8281 & 1910 & 1943 & 424 & 411 \\
        Bitcomp               & 11946 & 7886 & 8534 & 8153 & 1855 & 1919 & 414 & 404 \\
        CSR                   & -- & 15984 & -- & -- & -- & 2014 & -- & -- \\
        \midrule
        CompAct-2             & 11486 & 11486 & 10078 & 10078 & 2653 & 2653 & 485 & 485 \\
        CompAct-4             & 9955 & 9955 & 9238 & 9238 & 1976 & 1976 & 374 & 374 \\
        CompAct-8             & 9190 & 9189 & 8854 & 8854 & 1636 & 1636 & 324 & 324 \\
        \bottomrule
    \end{tabular}
\end{table}

\paragraph{Batch-size variation.}
Tables~\ref{tab:mlp-batch-size-relu2-50}, \ref{tab:mlp-batch-size-relu-50}, and~\ref{tab:mlp-batch-size-relu-90} complete the batch-size sweep of Table~\ref{tab:mlp-batch-size} for its three remaining activation and sparsity settings: $\relu^2$ at $\rho=0.5$, and $\relu$ at $\rho=0.5$ and $0.9$. The saving scales with sparsity and with the size of the dense baseline, not with the activation alone. At $32{,}000$ rows BitSparse reduces peak memory by $36\%$ for $\relu^2$ at $\rho=0.5$, $26\%$ for $\relu$ at the same sparsity, and $50\%$ for $\relu$ at $\rho=0.9$. The larger reduction for $\relu^2$ comes from the reduced transient storage used to compute activations in the backward pass from our optimized fused Triton kernel. The crossover at which sign-bit packing becomes the smallest configuration depends on both factors: it occurs at $8{,}000$ rows for both activation functions at $\rho=0.9$, and not at all at $\rho=0.5$. Runtime stays within about $1.5\%$ of dense storage at $32{,}000$ rows in all three settings. 

\begin{table}[htbp]
    \centering
    \caption{Effect of batch size on peak allocated memory and time per step for the $\relu^2$ MLP at sparsity $\rho=0.5$. Bold and italic values indicate the lowest and second-lowest results, respectively.}
    \label{tab:mlp-batch-size-relu2-50}
    \small
    \setlength{\tabcolsep}{3pt}
    \begin{tabular}{@{}rrrrrrrrr@{}}
        \toprule
        Batch & \multicolumn{4}{c}{Memory / MiB} & \multicolumn{4}{c}{Time / ms} \\
        \cmidrule(lr){2-5}\cmidrule(l){6-9}
        size & Dense & Checkpoint & BitSparse & BitSparse-sb & Dense & Checkpoint & BitSparse & BitSparse-sb \\
        \midrule
                32 & 3054 & \textbf{2879} & 2883 & \textit{2882} & \textbf{13} & 14 & 14 & \textit{13} \\
               128 & 3094 & \textbf{2899} & 2912 & \textit{2911} & \textbf{19} & 22 & \textit{19} & 20 \\
               512 & 3264 & \textbf{2982} & 3033 & \textit{3028} & \textit{64} & 75 & \textbf{64} & 67 \\
              2000 & 3932 & \textbf{3294} & 3492 & \textit{3471} & \textit{242} & 283 & \textbf{241} & 244 \\
              8000 & 7107 & \textbf{4704} & 5337 & \textit{5253} & \textbf{925} & 1088 & \textit{927} & 938 \\
             32000 & 19811 & \textbf{10704} & 12729 & \textit{12394} & \textit{3636} & 4270 & \textbf{3618} & 3660 \\
        \bottomrule
    \end{tabular}
\end{table}

\begin{table}[htbp]
    \centering
    \caption{Effect of batch size on peak allocated memory and time per step for the $\relu$ MLP at sparsity $\rho=0.5$. Bold and italic values indicate the lowest and second-lowest results, respectively, including ties.}
    \label{tab:mlp-batch-size-relu-50}
    \small
    \setlength{\tabcolsep}{3pt}
    \begin{tabular}{@{}rrrrrrrrr@{}}
        \toprule
        Batch & \multicolumn{4}{c}{Memory / MiB} & \multicolumn{4}{c}{Time / ms} \\
        \cmidrule(lr){2-5}\cmidrule(l){6-9}
        size & Dense & Checkpoint & BitSparse & BitSparse-sb & Dense & Checkpoint & BitSparse & BitSparse-sb \\
        \midrule
                32 & 3054 & \textbf{2878} & 2883 & \textit{2882} & \textbf{12} & 15 & 13 & \textit{13} \\
               128 & 3094 & \textbf{2894} & 2912 & \textit{2911} & \textbf{18} & 21 & \textit{19} & 19 \\
               512 & 3264 & \textbf{2960} & 3033 & \textit{3028} & \textbf{62} & 73 & \textit{64} & 65 \\
              2000 & 3889 & \textbf{3212} & 3490 & \textit{3470} & \textbf{234} & 273 & \textit{236} & 239 \\
              8000 & 6451 & \textbf{4376} & 5334 & \textit{5251} & \textbf{887} & 1037 & \textit{905} & 912 \\
             32000 & 17186 & \textbf{9392} & 12717 & \textit{12383} & \textbf{3478} & 4069 & \textit{3530} & 3571 \\
        \bottomrule
    \end{tabular}
\end{table}

\begin{table}[htbp]
    \centering
    \caption{Effect of batch size on peak allocated memory and time per step for the $\relu$ MLP at sparsity $\rho=0.9$. Bold and italic values indicate the lowest and second-lowest results, respectively.}
    \label{tab:mlp-batch-size-relu-90}
    \small
    \setlength{\tabcolsep}{3pt}
    \begin{tabular}{@{}rrrrrrrrr@{}}
        \toprule
        Batch & \multicolumn{4}{c}{Memory / MiB} & \multicolumn{4}{c}{Time / ms} \\
        \cmidrule(lr){2-5}\cmidrule(l){6-9}
        size & Dense & Checkpoint & BitSparse & BitSparse-sb & Dense & Checkpoint & BitSparse & BitSparse-sb \\
        \midrule
          32 & 3054 & \textbf{2878} & 2879 & \textbf{2878} & \textbf{12} & 13 & \textit{12} & 12 \\
         128 & 3094 & \textbf{2894} & \textit{2895} & \textit{2895} & \textbf{18} & 21 & \textit{19} & 19 \\
         512 & 3264 & \textbf{2960} & 2966 & \textit{2965} & \textbf{61} & 72 & \textit{63} & 63 \\
        2000 & 3889 & \textbf{3212} & 3233 & \textit{3230} & \textbf{232} & 271 & \textit{235} & 237 \\
        8000 & 6451 & 4376 & \textit{4307} & \textbf{4293} & \textbf{893} & 1040 & \textit{904} & 908 \\
       32000 & 17186 & 9392 & \textit{8607} & \textbf{8555} & \textbf{3482} & 4075 & \textit{3527} & 3546 \\
        \bottomrule
    \end{tabular}
\end{table}

\FloatBarrier
\subsection{Codec implementation ablation}
\label{app:codec-implementation-ablation}

Table~\ref{tab:codec-implementation-ablation} compares encoding and decoding with a pure PyTorch implementation of BitSparse and our custom Triton kernels. Both versions do the same operation, producing bit-identical outputs. The PyTorch version is optimized extensively using our LLM agent to ensure a fair comparison. Both operate on a $65{,}536\times8{,}192$ BF16 tensor ($1$ GiB in dense form). We report peak allocated GPU memory and average runtime for each operation. 

\begin{table}[htbp]
    \centering
    \caption{BitSparse codec implementation ablation on a $65{,}536\times8{,}192$ BF16 tensor. The baseline implements BitSparse in pure PyTorch; the proposed implementation uses custom Triton kernels. Lower peak memory and runtime are better.}
    \label{tab:codec-implementation-ablation}
    \small
    \begin{tabular}{llrr}
        \toprule
        Operation & Implementation & Peak memory / MiB & Time / ms \\
        \midrule
        Encode & Pure PyTorch & 2560.00 & 28.010 \\
               & Triton (BitSparse) & \textbf{1600.50} & \textbf{3.690} \\
        \midrule
        Decode & Pure PyTorch & 2048.53 & 17.168 \\
               & Triton (BitSparse) & \textbf{1600.50} & \textbf{2.170} \\
        \bottomrule
    \end{tabular}
\end{table}

\FloatBarrier
\subsection{Activation backward ablation}
\label{app:relu-backward-ablation}

Table~\ref{tab:relu-backward-ablation} isolates the pointwise activation derivative. The BF16 input shares the size of our MLP tests ($16{,}000$ rows and $21{,}504$ features) and has around 50\% sparsity. The PyTorch baseline is optimized to use fused inplace operations whenever possible for fast speed and as low a memory overhead as possible. Meanwhile, BitSparse fuses reconstruction of the required values and the derivative into an in-place update, using compact BF16 values. Peak allocated memory includes the gradient, retained activation representation, and any backward temporaries. Results are averaged over $32$ iterations after $5$ warmup iterations. 

\begin{table}[htbp]
    \centering
    \caption{Activation backward ablation on a $16{,}000\times21{,}504$ BF16 tensor with approximately $50\%$ sparsity. Peak memory includes the saved representation, gradient, and backward temporaries; time covers only the activation-derivative computation. Lower is better.}
    \label{tab:relu-backward-ablation}
    \small
    \begin{tabular}{llrr}
        \toprule
        Activation & Method & Peak memory / MiB & Time / ms \\
        \midrule
        $\relu$ & PyTorch & 1312.50 & 2.467 \\
                 & BitSparse & \textbf{698.25} & \textbf{2.043} \\
        \midrule
        $\relu^2$ & PyTorch & 2625.00 & 4.200 \\
                   & BitSparse & \textbf{1025.70} & \textbf{2.552} \\
        \bottomrule
    \end{tabular}
\end{table}

\FloatBarrier
\subsection{Shared value buffer ablation}
\label{app:buffer-ablation}

Table~\ref{tab:buffer-ablation} compares BitSparse with and without the shared value buffer described in Section~\ref{sec:method-buffer}. The comparison uses the same eight-block MLP setup as the main benchmark (Appendix~\ref{app:mlp-training}), with $\relu^2$ activations at sparsity $\rho=0.5$, and varies the input size from $32$ to $1024$ tokens. The buffer reduces runtime most at the smallest input size, from $12.563$ to $11.904$ ms at $32$ tokens. At $256$--$1024$ tokens, the runtime differences are below $0.4\%$ in magnitude, with a small slowdown at $256$ tokens rather than a consistent improvement.

\begin{table}[htbp]
    \centering
    \caption{Effect of the shared value buffer on BitSparse runtime as input size increases. Times retain three decimal places to show the small differences between configurations. Lower is better.}
    \label{tab:buffer-ablation}
    \small
    \begin{tabular}{rrr}
        \toprule
        Tokens & Without buffer / ms & With buffer / ms \\
        \midrule
        32   & 12.563  & 11.904 \\
        64   & 17.622  & 17.136 \\
        128  & 20.150  & 19.954 \\
        256  & 37.973  & 38.097 \\
        512  & 66.539  & 66.308 \\
        1024 & 130.807 & 130.518 \\
        \bottomrule
    \end{tabular}
\end{table}

\FloatBarrier
\subsection{Layer-count trade-offs}
\label{app:mlp-layer-pareto}

Figures~\ref{fig:mlp-layer-pareto-relu2-50}, \ref{fig:mlp-layer-pareto-relu-50}, and~\ref{fig:mlp-layer-pareto-relu-90} give the three remaining activation and sparsity settings for the layer-count sweep of Figure~\ref{fig:mlp-layer-pareto}.

Whenever at least one layer is compressed or checkpointed, BitSparse completes the step faster than checkpointing, because checkpointing pays for recomputation while BitSparse adds only a codec pass. Checkpointing nevertheless attains the lowest VRAM in the two $\rho=0.5$ settings, and the margin there is large: for $\relu$ at $8$ layers, BitSparse uses $29\%$ more memory than checkpointing at $\rho=0.5$, but $5\%$ less at $\rho=0.9$. Both $\rho=0.9$ settings therefore cross below checkpointing at the largest layer count.

The steepest fall in VRAM occurs at the final layer count, and it is far sharper for $\relu^2$ than for $\relu$: moving from $7$ to $8$ compressed layers lowers BitSparse's VRAM by $23\%$ at $\rho=0.9$ and $15\%$ at $\rho=0.5$ for $\relu^2$, but by only $7\%$ and $1\%$ for $\relu$. Sign-bit packing tracks BitSparse closely in every setting, saving slightly more memory at a small cost in time.

\begin{figure}[htbp]
    \centering
    \begin{minipage}[t]{0.48\linewidth}
        \centering
        \includegraphics[width=\linewidth]{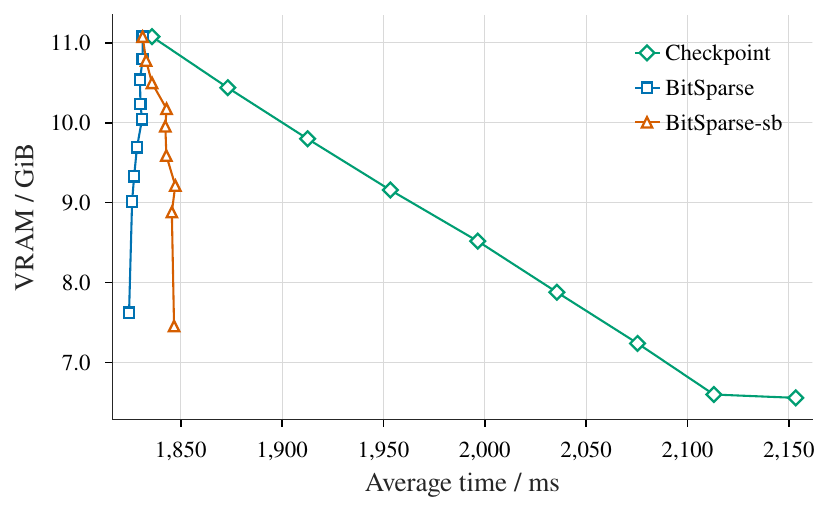}
        \caption{Memory--runtime trade-off for the $\relu^2$ MLP at sparsity $\rho=0.5$. Conventions follow Figure~\ref{fig:mlp-layer-pareto}.}
        \label{fig:mlp-layer-pareto-relu2-50}
    \end{minipage}\hfill
    \begin{minipage}[t]{0.48\linewidth}
        \centering
        \includegraphics[width=\linewidth]{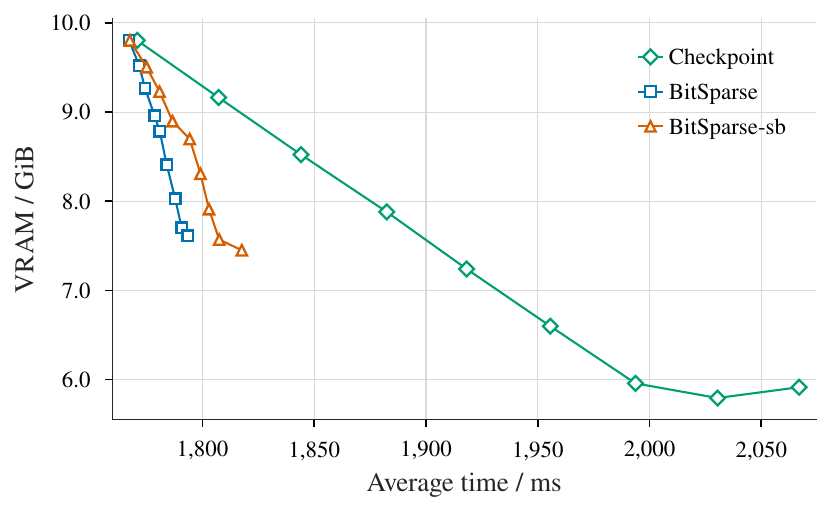}
        \caption{Memory--runtime trade-off for the $\relu$ MLP at sparsity $\rho=0.5$. Conventions follow Figure~\ref{fig:mlp-layer-pareto}.}
        \label{fig:mlp-layer-pareto-relu-50}
    \end{minipage}

    \medskip
    \begin{minipage}[t]{0.48\linewidth}
        \centering
        \includegraphics[width=\linewidth]{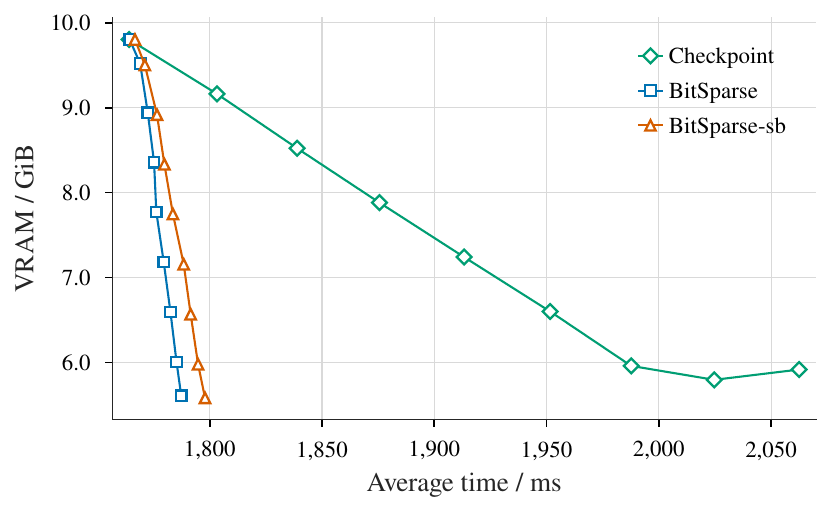}
        \caption{Memory--runtime trade-off for the $\relu$ MLP at sparsity $\rho=0.9$. Conventions follow Figure~\ref{fig:mlp-layer-pareto}.}
        \label{fig:mlp-layer-pareto-relu-90}
    \end{minipage}
\end{figure}

\end{document}